\documentclass[12pt]{article} 
\usepackage{natbib, amsthm}

\usepackage{fullpage}
\usepackage{bbm}
\usepackage{times}
\usepackage{graphicx,color}
\usepackage{array,float}
\usepackage{url}
\usepackage{amstext,amssymb,amsmath}
\usepackage{hyphenat}
\usepackage{verbatim}
\usepackage{bm}
\usepackage{paralist}
\usepackage{ulem}
\usepackage{todonotes}
\usepackage{paralist}\usepackage{bbm}
\usepackage{wrapfig}
\usepackage{hyperref}
\usepackage[noend]{algorithmic}
\usepackage{algorithm}
\usepackage{xparse,etoolbox}
\usepackage{threeparttable}
\usepackage{bbm}
\usepackage{enumitem}
\usepackage{dsfont}

\newtheorem{lem}{Lemma}[section]
\newtheorem{theorem}{Theorem}
\newtheorem{remark}[theorem]{Remark}

\newtheorem{thm}[lem]{Theorem}
\newtheorem{infthm}[lem]{Informal Theorem}
\newtheorem{cor}[lem]{Corollary}

\newtheorem{defn}[lem]{Definition}

\newtheorem{claim}[lem]{Claim}

\renewcommand{\paragraph}[1]{\vspace{3pt}\noindent\textbf{#1}}

\newcommand{\cX}{\ensuremath{\mathcal{X}}}
\newcommand{\cY}{\ensuremath{\mathcal{Y}}}
\newcommand{\bcY}{\ensuremath{\overline{\mathcal{Y}}}}
		
\newcommand{\hD}{\widehat D}

\newcommand{\hs}{\hat{s}}
\newcommand{\hy}{\hat{y}}
\newcommand{\hh}{\hat{h}}
\newcommand{\tlh}{\tilde{h}}

\newcommand{\flag}{\mathsf{Flag}}

\newcommand{\lpriv}{\mathsf{PrvLearn}}
\newcommand{\hpv}{h_{\mathsf{priv}}}

\newcommand{\hpriv}{\mathsf{H}_{\mathsf{Priv}}}

\newcommand{\ppac}{\mathsf{Priv}}

\newcommand{\aux}{\mathsf{AUX}}

\newcommand{\hist}{\mathsf{Hist}}
\newcommand{\shist}{\widehat{\mathsf{Hist}}}

\newcommand{\hcP}{\widehat{\mathcal{P}}}

\newcommand{\slc}{\mathsf{SL-Quer}}

\DeclareMathOperator*{\argmax}{arg\,max}

\newcommand{\pr}[2]{\underset{#1}{\mathbb{P}}\left[ #2 \right]}
\newcommand{\ex}[2]{\underset{#1}{\mathbb{E}}\left[ #2 \right]}
\newcommand{\var}[2]{\underset{#1}{\mathbf{Var}}\left[ #2 \right]}

\newcommand{\norm}[1]{\|#1\|}

\newcommand{\eps}{\epsilon}

\newcommand{\cA}{\mathcal{A}}
\newcommand{\cI}{\mathcal{I}}

\newcommand{\cD}{\mathcal{D}}

\newcommand{\cG}{\mathcal{G}}
\newcommand{\cH}{\mathcal{H}}

\newcommand{\bw}{\mathbf{w}}

\newcommand{\cP}{\mathcal{P}}

\newcommand{\ind}{{\mathbf{1}}}

\newcommand{\cS}{\mathcal{S}}

\newcommand{\cF}{\mathcal{F}}

\newcommand{\re}{\mathbb{R}}
\newcommand{\R}{\mathcal{R}}

\newcommand{\cB}{\mathcal{B}}

\newcommand{\cQ}{\mathcal{Q}}

\newcommand{\tD}{\tilde{D}}

\newcommand{\ignore}[1]{}

\newcommand{\cR}{\mathcal{R}}
\newcommand{\A}{\mathcal{A}}

\newcommand{\cC}{\mathcal{C}}

\newcommand{\genhist}{\mathsf{GenHist}}

\newcommand{\err}{\mathsf{err}}

\newcommand{\oa}{\mbox{on-average}}

\newcommand{\dist}{{\sf dist}}
\newcommand{\thr}{\Gamma}
\newcommand{\ds}{{\sf dist}\,}
\newcommand{\wds}{\widehat{\ds}}
\newcommand{\out}{{\sf out}}

\newcommand{\blocktheorem}[1]{%
	\csletcs{old#1}{#1}
	\csletcs{endold#1}{end#1}
	\RenewDocumentEnvironment{#1}{o}
	{\par\addvspace{1.5ex}
		\noindent\begin{minipage}{\textwidth}
			\IfNoValueTF{##1}
			{\csuse{old#1}}
			{\csuse{old#1}[##1]}}
		{\csuse{endold#1}
		\end{minipage}
		\par\addvspace{1.5ex}}
}

\blocktheorem{theo}

\begin{document}
	\date{}
	\title{Model-Agnostic Private Learning via Stability}

\author{
	Raef Bassily\thanks{Department of Computer Science \& Engineering, The Ohio State University. \texttt{bassily.1@osu.edu}}
	\and Om Thakkar\thanks{Department of Computer Science, Boston University. \texttt{omthkkr@bu.edu}}\and
	Abhradeep Thakurta\thanks{Department of Computer Science, University of California Santa Cruz. \texttt{aguhatha@ucsc.edu}}
	}

	\maketitle

\begin{abstract}

We design differentially private learning algorithms that are agnostic to the learning model. Our algorithms are interactive in nature, i.e., instead of outputting a model based on the training data, they provide predictions for a set of $m$ feature vectors that arrive online. We show that, for the feature vectors on which an ensemble of models (trained on random disjoint subsets of a dataset) makes consistent predictions (i.e., the models sufficiently agree on a given prediction value), there is \emph{almost no-cost of privacy} in generating accurate predictions for those feature vectors. To that end, we provide a novel coupling of the distance to instability framework with the sparse vector technique.


We provide algorithms with formal privacy and utility guarantees for both binary/multi-class classification, and soft-label classification (where the label is a score $\in[0,1]$). For binary classification in the standard (agnostic) PAC model, we show how to bootstrap from our privately generated predictions to construct a \emph{computationally efficient} generic private learner that outputs a final accurate hypothesis from the given concept class. In particular, our construction -- to the best of our knowledge -- is the first \emph{computationally efficient} construction for a \emph{label-private} learner. We prove sample complexity upper bounds for this setting in both the realizable and the non-realizable (agnostic) cases. As in non-private sample complexity bounds, the only relevant property of the concept class in our bounds is its VC dimension. 
For soft-label classification, our techniques are based on exploiting the stability properties of traditional learning algorithms, like stochastic gradient descent (SGD). We provide a new technique to boost the average-case stability properties of learning algorithms to strong (worst-case) stability properties, and then exploit them to obtain differentially private classification algorithms. In the process, we also show that a large class of SGD methods satisfy average-case stability properties, in contrast to a smaller class of SGD methods that are uniformly stable as shown in prior work. 

    
    
    
    
	\end{abstract}	
	\newpage
	\thispagestyle{empty}

	\section{Introduction}
\label{sec:intro}

We study the problem of classification of \emph{public} \emph{unlabeled} data $\{x_1,\cdots,x_m\}\subseteq\mathcal{X}$ when the learning algorithm $\cA$ is given a \emph{private} training dataset $D$. The goal is to enable answering as many classification queries as possible while ensuring the privacy of the training set. Most of the prior work on differentially private machine learning has focused on the ``once and for all'' approach, where a differentially  private algorithm is first used to train the private dataset, and then output a model (or a classifier) $\theta$ that can be later used be used for performing the desired learning tasks. A major issue with this approach is that the accuracy guarantee of such private learners (or equivalently, their sample complexity) usually suffer from an explicit dependence on the dimensionality of the problem (see, e.g., \citep{CMS11,KST12,ST13,BassilyST14}), or the size of the feature space, even for simple tasks like learning threshold functions \citep{bun2015differentially}. Such a dependency may limit the utility of this approach, especially in modern machine learning where data is high-dimensional, and models are complex and over-parameterized. 

To deal with this issue, we take a different approach. Instead of training an algorithm to output a ``safe'' model that can be used indefinitely, we ensure that the predictions made by the algorithm in response to a number of (classification) queries preserve differential privacy. This approach allows us to produce privatized predictions whose accuracy is very close to that of non-private learners as long as the number of queries $m$ is not too large. Roughly speaking, this approach allows to achieve accuracy that is \emph{almost the same} as that of non-private learners, at the expense of increasing the sample complexity (compared to that of the non-private learner) by a factor that does not depend on the model, or the complexity of the learning problem. It only depends on the number of queries (out of the $m$ queries) that are ``bad'' in some natural sense. Hence, the limiting factor becomes the number of those ``bad'' queries rather than the dimensionality, or the complexity of the problem. 
Moreover, this approach enables using \emph{any} non-private learner in a black-box fashion, with the offered utility guarantees being relative to those of the underlying non-private learner. Hence, our algorithms can be used in very general settings. Also, having a black-box access to any generic learner may be appealing from a practical standpoint as well. 

Another significant advantage of this approach is  that the privately generated labels (the answers to the $m$ queries) can then be used together with queried feature vectors to train a new model. Since differential privacy is closed under post-processing, the final model is safe for publication and can be used to answer subsequent classification queries indefinitely.  This technique has been commonly referred to as \emph{semi-supervised knowledge transfer}. In the context of privacy, this technique has been explored in \cite{hamm2016learning}. It has also been extended and empirically investigated in \cite{papernot2017semi}, but without formal guarantees. Using our constructions, we explicitly show  how to achieve this task efficiently, and prove non-trivial upper bounds on the sample complexity in the standard (agnostic) PAC model.



Prior theoretical work in this area \citep{CMS11,KST12,ST13,BassilyST14,TTZ15,wu2017bolt} has exclusively studied the setting where the underlying learning problem satisfies convexity properties. Since modern machine learning involves non-convex models, another motivation for this work is to provide differentially private learning algorithms with formal utility guarantees without invoking any convexity assumptions. Our algorithms and analyses solely rely on the stability properties of the underlying training procedure on the private dataset $D$. By stability, we mean that the behavior of the training procedure should not change significantly by adding (removing) a few entries to (from) the dataset $D$. 
Few recent works \citep{wang2015privacy,abadi2016deep,papernot2017semi,46614} have studied private non-convex learning, but they do not provide any formal utility guarantees. Analyzing non-convex learning algorithms with differential privacy is especially hard because in many cases, the robustness of the underlying learning models is not well-understood.

\paragraph{Algorithmic idea:} Suppose that a model class is independently trained $k$ times, each time with the same number of training examples given to the learning algorithm. Naturally, one would expect the corresponding output classifiers $\theta_1,\cdots,\theta_k$ to predict ``similarly'' on a new example from the same distribution. Using this idea, we design differentially private algorithms that are agnostic to the underlying learning problem. 

The constructions in 
\cite{papernot2017semi,46614} \emph{empirically} exhibit a similar behavior, but no formal guarantees are provided for it. Furthermore, \cite{ST13,papernot2017semi,46614} consider the setting where the final output is a discrete object (e.g., a set of features, or a hard label $\in\{0,1\}$). Our results are more general, and extend to the setting where the predictions are soft-labels (i.e., they can lie in the interval $[0,1]$) rather than being restricted to hard labels.  We provide a detailed comparison with the relevant prior work in Section ~\ref{app:comp}.

\paragraph{Our Techniques:} Our algorithms build on the sub-sample and aggregate framework of \cite{NRS07,ST13}, and are most related to the framework of \cite{ST13} which is based on the stability idea described earlier. We combine this framework with the sparse vector technique \citep{DRV10, HR10, DR14} to show that among $m$ classification queries, one only needs to ``pay the price of privacy'' for the queries that result in an unstable histogram for predictions by the classifiers $\theta_1, \ldots, \theta_k$.  




\subsection{Our Contributions}
\label{sec:ourContrib}
\paragraph{A private framework for stable online queries.} 
We provide a generic algorithm that extends the sample and aggregate framework to enable answering $m$ online queries on a private dataset. Our algorithm is guaranteed to answer all $m$ queries with high accuracy only if the number of \emph{unstable queries} (in the sense of \cite{ST13}) does not exceed a certain cutoff value $T$ which is given as input parameter. Our algorithm is guaranteed to satisfy differential privacy regardless of whether this condition is satisfied\footnote{We also provide a corrected proof for the privacy guarantee of the distance to instability framework in \cite{ST13}}. The utility on the other hand depends on this condition. In particular, our construction implies an upper bound on the sample complexity that scales $\approx \sqrt{T}$. Note this result does not contradict with the existing lower bounds of $\sqrt{m}$ (e.g., see \cite{DSSUV15, SU15}) since those lower bounds are applicable under a weaker set of assumptions. In particular, we can circumvent those bounds only in scenarios where there is a guarantee that the number of unstable queries does not exceed $T=o(m)$.

\paragraph{Privately answering binary classification queries.} In Section~\ref{sec:privClass-hard}, we show how to use our framework to construct generic private algorithm for answering binary classification queries on public data based on \emph{private} training data. Given a non-private (agnostic) PAC learner for a hypothesis class $\cH$ of VC-dimension $V$ and a private training set of $n$ examples, we can answer $m$ classification queries with a misclassification rate of $\tilde{O}\left({mV^2/n^2}\right)$ in the realizable case for sufficiently large $m$ , namely, for all $m\geq \Omega\left((n/V)^{4/3}\right)$. This implies that in the realizable setting, we can answer up to $\approx (n/V)^{4/3}$ queries with misclassification rate $\approx (V/n)^{2/3}$, i.e., only a factor of $\approx (n/V)^{1/3}$ worse than the optimal non-private miscalssification rate. Similarly, in the agnostic (non-realizable) case, we can achieve misclassification rate of $\tilde{O}\left(m^{1/3}V^{2/3}/n^{2/3}\right)$ for all $m\geq \Omega\left((n/V)^{4/5}\right)$. In particular, this implies that in the non-realizable setting, we can answer up to $\approx (n/V)^{4/5}$ queries with misclassification rate $\approx (V/n)^{2/5}$, i.e., only a factor of $\approx (n/V)^{1/10}$ worse than the optimal non-private miscalssification rate. Furthermore, our algorithm does not require any extra assumptions on the hypothesis class $\cH$ beyond what is assumed in the non-private setting. As mentioned earlier, the existing bounds in differentially private PAC learning \citep{BLR08, kasiviswanathan2011can,beimel2013characterizing, bun2015differentially} do not depend on $m$, but on the other hand, they either assume that the hypothesis class is finite, or they suffer from explicit dependence on the size of the domain even for simple classes with VC-dimension 1. 

\paragraph{Efficient label-private learning:} In Section~\ref{sec:label-priv}, we build on our results above to achieve a stronger objective than merely answering classification queries. In particular, we show how to bootstrap from our construction  a \emph{computationally efficient} private learner that publishes an accurate classifier, which can then be used to answer indefinite number of classification queries. In particular, our learner provide privacy guarantees in either one of the following scenarios: (i) a scenario where only the labels of the training set are considered private information, or (ii) a scenario where the entire training sample is private but the learner has access to public unlabeled data from the same distribution. We prove sample complexity upper bounds for both realizable and non-realizable settings in the standard (agnostic) PAC model. Our bounds can be summarized as follows:

\begin{infthm}[Corresponding to Theorems~\ref{thm:ppac},~\ref{thm:agppac}]
Let $\cH$ be a hypothesis class of VC-dimension $V$. Let $0<\beta\leq \alpha < 1$. Let $n$ be the size of the private training set and let $m=O\left(\frac{V+\log(1/\beta)}{\alpha^2}\right)$ be the number of unlabeled public feature vectors. There is an efficient $(\epsilon, \delta)$-differentially private algorithm (instantiation of Algorithm~\ref{Alg:subSamp}) that generates labels for all $m$ feature vectors, then use the new labeled set of $m$ points as another training set for a non-private (agnostic) PAC learner, which finally outputs a hypothesis $\hh\in\cH$ (see Algorithm~\ref{Alg:ppac}). The following guarantees hold for the realizable and the agnostic settings: 
\begin{itemize}
\item Realizable case: if $n=\tilde{O}\left(V^{3/2}/\alpha^{3/2}\right),$ then, w.p. $\geq 1-\beta,$ we have $\err(\hh; \cD)=O(\alpha)$.

\item Agnostic case: if $n=\tilde{O}\left(V^{3/2}/\alpha^{5/2}\right),$ then, w.p. $\geq 1-\beta,$ we have $\err(\hh; \cD)=O(\alpha+\gamma)\,$ (where $\gamma=\min\limits_{h\in\cH}\err(h; \cD)$).
\end{itemize} 
\end{infthm}

Our bound in the realizable case is only a factor of $\tilde{O}(\sqrt{V/\alpha})$ worse than the optimal non-private sample complexity, where $V$ is the VC dimension of the concept class and $\alpha$ is the achievable accuracy. In the agnostic case, our bound exhibits the same connection to the optimal non-private sample complexity, however, we note that the accuracy of the output hypothesis in this case may have a suboptimal dependency (by a small constant factor) on $\gamma\triangleq \min\limits_{h\in\cH}\err(h; \cD)$.

Label-private learning has been explored before in \cite{chaudhuri2011sample} and \cite{BeimelNS16}. Both works have only considered pure, i.e., $(\epsilon, 0)$, differentially private learners and their constructions are computationally inefficient. We give a more detailed comparison with previous work in Section~\ref{sec:label-priv}.

\paragraph{Privately answering classification queries for soft-label classification.} 
In Section~\ref{sec:privClass-soft}, we consider the private soft-label classification problem, where for each feature vector $x\in\mathcal{X}$ the objective is to output a label in the range $[0,1]$. Soft-labels (or, soft predictions) are especially useful in the case of ranking problems like movie or product recommendation \citep{toscher2009bigchaos}, and  advertisement ranking \citep{mcmahan2013ad}. One can view the soft-labels as estimates for the conditional probability of the underlying true hard label given a feature vector $x\in\cX$. One immediate approach for the soft-label case is to discretize the interval $[0,1]$ and use the same framework for the binary (or the multi-label) classification. Since our accuracy guarantees do not explicitly depend on the number of classes (in multi-label classification), therefore we can have arbitrary fine-grained discretization, as long as there is sufficient agreement (up to the discretization width) among the collection of classifiers produced by the underlying (non-private) learner. We show that if such learner satisfies a notion of on-average stability (weaker than the standard notion of uniform stability), then one can show that the soft predictions of those classifiers will also have sufficient concentration, up to the discretization width. We show this via a novel argument that uses an Efron-Stein style inequality \citep{steele1986efron}. 

We also show that standard algorithms like stochastic gradient descent (SGD) satisfy such notion of stability. A technical issue that arises with this discretization approach is that, although we can show that the ensemble predictions are concentrated within a small region, we do not know \emph{where exactly does that region lie.} We address this by designing the algorithm such that there is no cost of privacy (in terms of the privacy parameter $\epsilon$) for those feature vectors (queries) whose concentration occurs close to zero, or close to one. Intuitively, this means that for feature vectors on which the underlying learner has relatively high confidence about the nature of the underlying hard label, there is no cost of privacy. If the concentration is not close to either zero or one, then we use the shifted discretization trick from \cite{DL09} to answer those queries, but do pay a price of privacy in this case. Moreover, we prove a lemma which shows that our private algorithm amplifies the confidence of the underlying non-private learner. Our main theorem in this section can be informally stated as follows. 

\begin{infthm}[Corresponding to Theorem~\ref{thm:mainUtil}]
Let $m$ be the number of soft-label queries. Let $T \in [m]$, $\nu\in [0 ,1/2]$. Suppose that the underlying (non-private) learner is $\alpha$-on average stable (in the sense of Definition~\ref{defn:unif-stab}) with respect to input dataset of $n'$ i.i.d. examples. Suppose that the number of feature vectors (queries) for which the expected value of the soft prediction generated via that learner is $\in (\nu, 1-\nu)$ does not exceed $T$. Then, given an input dataset of size $n\approx n'\,\sqrt{T}\,/\epsilon$ (ignoring log factors in $m , 1/\delta$), with high probability, our private algorithm (Algorithm~\ref{Alg:privClass-soft}) answers all $m$ queries with soft predictions that are  within $\approx \nu/2 + \alpha\sqrt{n'}$ from the expected value of the prediction of the non-private learner. 
\end{infthm}

When the expected value of the soft prediction generated via a learner on feature vector $x$ is either $\leq \nu$ or $\geq 1-\nu$, we say that learner has $\nu$-quality on $x$. The above theorem provides non-trivial guarantees when (i) the underlying non-private learner is $o(1/\sqrt{n'})$-on average stable (where $n'$ is the size of its input sample), and (ii) the number of queried feature vectors for which the learner has $\nu$-quality is $\geq m-T$ where $T\ll m$. In particular, when these two conditions are satisfied, then with only an extra factor of $\approx \sqrt{T}/\epsilon$ in the size of the input sample, with high probability our private algorithm yields soft predictions that are almost as accurate as the expected predictions of the non-private learner.

For the first condition, we note that popular algorithms such as SGD satisfy this condition under standard assumptions on the loss function used in training (e.g., see \cite[Theorem 3.9]{hardt2015train}). The second condition is typically satisfied in scenarios with relatively low \emph{label noise}. To elaborate, let $p(y|x), y\in \{0,1\}$ denote the true conditional probabilities of the \emph{hard} label given a queried feature $x$. Since a soft-label generated by a learner for a feature vector $x$ can be viewed as an estimate for $p(1|x)$, in scenarios where $\lvert p(0|x) - p(1|x)\rvert$ is relatively large with high probability over the choice of $x$ (i.e., when the Bayes risk is small), a learner that generates sufficiently accurate soft-labels (i.e., good estimates for $p(1|x)$) will satisfy the above notion of quality with small $\nu$ for all except a small fraction of the queried feature vectors. Hence, in such settings, our second condition will be satisfied with $T\ll m$. (For more details,  see discussion after Definition \ref{defn:wquality} in  
Section \ref{sec:privClass-soft}.)

\paragraph{Revisiting the stability of SGD.} We provide a result showing that SGD satisfies our notion of $\oa$ stability mentioned above. The work by \cite{hardt2015train} shows that randomized SGD (particularly via randomly permuting the data, or sampling with replacement) satisfies uniform stability, which is stronger than our notion of $\oa$ stability. On the other hand, our result for $\oa$ stability of SGD holds even when running SGD does \emph{not} involve randomization via permutation or sampling, and hence does not follow directly from the results of \cite{hardt2015train}. For example, it applies even to the (deterministic) one-pass version of SGD. Our proof is based on a reduction from $\oa$ stability of any standard SGD method to \emph{uniform stability} of SGD that uses a \emph{random permutation}. We think the proof technique could be of independent interest. The result is based on a very simple argument that involves manipulating random variables and their expectations, and uses simple properties of i.i.d. sequences.

\subsection{Comparison with Relevant Prior Work}\label{app:comp}


\paragraph{Sparse-vector based constructions:} \cite{HR10,gupta2012iterative} showed that it is possible to answer $m$ linear queries accurately on a dataset with the privacy cost having a dependence logarithmically in $m$. They obtained the guarantees by using a variant of the sparse vector technique \citep{DRV10, DR14}. The guarantees we provide are incomparable to these results as the query classes under consideration are different. 

\paragraph{Stability and Sample-and-aggregate based constructions:} \cite{ST13} introduced the distance to instability framework, and combined it with the sub-sample and aggregate framework \citep{NRS07} to provide the first non-trivial error bounds for high-dimensional LASSO estimators. We extend their framework, combining it with the sparse vector technique to answer a large number of classification queries while incurring a privacy cost only for the ``unstable" queries among them. 

\paragraph{Application in knowledge transfer in deep learning:} \cite{papernot2017semi,46614} use a similar idea of sub-sample and aggregate to train deep learning models. Our work differs from these works in a couple of aspects: i) the aggregation scheme used in \cite{papernot2017semi} is a variant of the exponential mechanism \citep{MT07}, the one used in \cite{46614} is a combination of the Report Noisy Max algorithm \citep{DR14} and the sparse vector technique, whereas we use the distance to instability framework of \cite{ST13} combined with sparse vector; ii) both the prior works do not provide any formal utility guarantees, whereas we give provable accuracy guarantees for our algorithm. \cite{46614} empirically observe that if there is agreement among the ensemble of classifiers for a given classification query, then the cost of privacy for that query is very low. Furthermore, they use the sparse vector technique to exploit this observation to answer a large number of queries while keeping the overall privacy budget small. We (independently) make a similar observation analytically, and formally quantify the overall ``privacy cost''. Additionally, the techniques from both the prior works apply only for classification queries, whereas our framework incorporates soft-label classification as well. 

\paragraph{Label-private learning:} \cite{chaudhuri2011sample} gave upper and lower bounds on the sample complexity label-private learning in terms of the doubling dimension. The work of \cite{BeimelNS16} showed that the sample complexity can be characterized in terms of the VC dimension. Both works have considered only pure differentially private learners, and their constructions are computationally inefficient. Also, \cite{BeimelNS16} considered only the realizable case, where they gave an upper bound on the sample complexity, which is only a factor of $O(1/\alpha)$ worse than the optimal non-private bound for the realizable case. Comparing to \cite{BeimelNS16}, our work gives a \emph{computationally efficient} construction of a label-private PAC learner with sample complexity upper bound that is only a factor of $\tilde{O}\left(\sqrt{V/\alpha}\right)$ worse than the optimal non-private bound. On the other hand, our construction satisfies \emph{approximate} rather than pure differential privacy.  Moreover, we prove a sample complexity upper bound for the agnostic (non-realizable) case. Our bound in the agnostic case exhibits the same connection to the optimal non-private sample complexity, (i.e., only a factor of $\tilde{O}\left(\sqrt{V/\alpha}\right)$ larger), however, the accuracy of the final output classifier in this case may have a suboptimal dependency (by a small constant factor) on $\gamma\triangleq \min\limits_{h\in\cH}\err(h; \cD)$ (the best possible accuracy achieved for the given concept class $\cH$).

    \section{Background and Preliminaries}
\label{sec:background}

In this section, we formally define the notation, important definitions, and the existing results used in this work. 

 We denote the data universe by $U$, and an $n$-element dataset by $D = \{z_1, z_2, \ldots, z_n\}$. For any two datasets $D, D' \in U^*$, we denote the symmetric difference between them by $D\Delta D'$. For classification tasks, we use $\cX$ to denote the space of feature vectors, and $\bcY$ to denote the set of labels.  
 Thus, $U = \cX\times \cY$ in this case, and each data element is denoted as $(x,y)$. First, we provide a definition of PAC learning (used in Section~\ref{sec:privClass-hard}).
 
 \begin{defn}[Agnostic Probably Approximately Correct (PAC) learner (\cite{kearns1994introduction})]
Let $\cD$ be a distribution defined over the space of feature vectors and labels $U=\cX\times \cY$. Let $\cH$ be a hypothesis class with each $h\in\cH$ is a mapping $h:\cX\to\cY$. We say an algorithm $\Theta:U^*\to\cH$ is an Agnostic PAC learner for $\cH$ if it satisfies the following condition: For every $\alpha, \beta \in (0, 1)$, there is a number $n=n(\alpha, \beta)\in \mathbb{N}$ such that when $\Theta$ is run on a dataset $D$ of $n$ i.i.d. examples from $\cD$, then with probability $1-\beta$ (over the randomness of $D$) it outputs a hypothesis $h_D$ with $\err(h_D; \cD)\leq \gamma + \alpha$, where $\err(h; \cD)\triangleq \pr{(x, y)\sim \cD}{h_D(x)\neq y}$ and $\gamma \triangleq \min\limits_{h\in \cH}\err(h; \cD)$. 
\label{def:agPAC}
\end{defn}


We will also use the following parametrized version of the above definition. 

\begin{defn}[$(\alpha, \beta, n)$-learner for a class $\cH$]
Let $\alpha, \beta \in (0, 1)$ and $n\in \mathbb{N}$. An algorithm $\Theta$ is $(\alpha, \beta, n)$-(agnostic) PAC learner if, given an input dataset $D$ of $n$ i.i.d. examples from the underlying unknown distribution $\cD$, with probability $1-\beta$, it outputs a hypothesis $h_D$ with $\err(h_D; \cD) \leq \gamma + \alpha$ (where $\gamma$ is defined as in Definition~\ref{def:agPAC} above).
\label{defn:learner}
\end{defn}

 Next, we define the notion of differential privacy, state some of its properties, and describe some of the mechanisms and frameworks for achieving differential privacy that we use in this work. We start by providing the definition of \emph{approximate} differential privacy.

\begin{defn}[$(\epsilon, \delta)$-Differential Privacy (\cite{DKMMN06,DMNS06})]
A (randomized) algorithm $M$ with input domain $U^*$ and output range $\cR$ is $(\eps,\delta)$-differentially private if for all pairs of datasets $D,D'\in U^*$ s.t. $|D\Delta D'|= 1$, and every measurable $S \subseteq \cR$, we have with probability at least $1-\delta$ over the coin flips of $M$ that: 
$$\Pr \left(M(D) \in S \right) \leq e^\eps \cdot \Pr \left(M(D') \in S \right).$$
\end{defn} 
When $\delta=0$, it is known as \emph{pure} differential privacy, and parameterized only by $\eps$ in this case. 

An important advantage of differential privacy is that it is closed under \emph{post-processing}, which we describe next. 

\begin{lem}[Post-processing (\cite{DMNS06})]
\label{lem:post}
If a mechanism $M: U^* \rightarrow Y$ is $(\epsilon,\delta)$-differentially private, then for any function $f:Y \rightarrow Y'$, we have that $f\circ M$ is also $(\epsilon,\delta$)-differentially private.
\end{lem}

One of the most common techniques for achieving pure differential privacy is the Laplace mechanism, for which we first define the global sensitivity of a function.

\begin{defn}[Global sensitivity]
\label{defn:sen}
A function $f: U^* \rightarrow \mathbb{R}$ has global sensitivity $\gamma$ if $$\max\limits_{\substack{D,D' \in U^*:\\|D\Delta D'|= 1}} \| f(D) - f(D')\|_1 = \gamma.$$
\end{defn}

\begin{lem}[Laplace mechanism (\cite{DMNS06})]
\label{lem:lap}
If a function $f: U^n \rightarrow \mathbb{R}^p$ has global sensitivity $\gamma$, then the mechanism $M$, which on input $D \in U^n$ outputs $f(D) + b$, where $b \sim Lap\left(\frac{\gamma}{\epsilon} \right)^p$, satisfies $\eps$-differential privacy. Here, $Lap(\lambda)^p$ denotes a vector of $p$ i.i.d. samples from the Laplace distribution $Lap(\lambda)$.
\end{lem}

\subsection{The Sparse Vector Technique}
Here, we describe the Sparse vector technique, and provide the privacy and utility guarantees for it. Sparse vector allows answering a set of queries in an online setting, where a cost for privacy is incurred only if the answer to a query falls near or below a predetermined threshold.  We denote the set of queries by $\cQ=\{q_1,\cdots,q_m\}$, where every $q_i:U^*\to\re$, and has global sensitivity at most one. We provide a pseudocode for the technique in Algorithm~\ref{Alg:sparseVec}. Next, we provide the privacy and accuracy guarantees for Algorithm~\ref{Alg:sparseVec}. 

\begin{algorithm}[h]
	\caption{$\A_{\sf sparseVec}$: Sparse vector technique}
	\begin{algorithmic}[1]
		\REQUIRE dataset: $D$, query set $\cQ=\{q_1,\cdots,q_m\}$, privacy parameters $\epsilon,\delta >0$, unstable query cutoff: $T$, threshold: $w$
		\STATE $c\leftarrow 0$, $\lambda\leftarrow \sqrt{32T\log(1/\delta)}/\epsilon$, and $\widehat w\leftarrow w+{\sf{Lap}}(\lambda)$ 
		\FOR{$q\in \cQ$ and $c\leq T$}
		\STATE $\widehat q\leftarrow q + {\sf Lap}(2\lambda)$
		\STATE {\bf If} $\widehat q > \widehat{w}$, {\bf then} , output $\top$, {\bf else} output $\bot$, and set $\widehat w\leftarrow w+{\sf Lap}(\lambda)$, $c\leftarrow c + 1$ 
		\ENDFOR
	\end{algorithmic}
	\label{Alg:sparseVec}
\end{algorithm}

\begin{thm}[Privacy guarantee (\cite{DRV10,HR10,DR14})]
	Algorithm \ref{Alg:sparseVec} is $(\epsilon,\delta)$-differentially private.
	\label{thm:spPriv}
\end{thm}

\begin{thm}[Accuracy guarantee (\cite{DRV10,HR10,DR14})]
	For $\alpha=\log(2mT/\beta)\sqrt{512T\log(1/\delta)}/\epsilon$, and any set of $m$ queries $q_1,\cdots,q_m$, define the set  $L(\alpha)=\left\{i:q_i(D) \leq w + \alpha \right\}$. If $|L(\alpha)|\leq T$, then we have the following w.p. at least $1-\beta$: $\forall i\not\in L(\alpha)$ Algorithm \ref{Alg:sparseVec} outputs $\top$.
	\label{thm:spAcc}
\end{thm}

	\section{Privately Answering Stable Online Queries}
\label{sec:seqComp}

In this section, we design a generic framework that allows answering a set of queries on a dataset while preserving differential privacy, and only incurs a privacy cost for the queries that are unstable. 
 
 \subsection{Distance to Instability Framework}
\label{sec:distInstab}

First, we describe the distance to instability framework from \cite{ST13} that releases the exact value of a function on a dataset while preserving differential privacy, provided the  function is sufficiently \emph{stable} on the dataset. We define the notion of stability first, and provide the pseudocode for a private estimator for any function via this framework in Algorithm~\ref{Alg:kStab}. 

\begin{defn}[$k$-stability \citep{ST13}]
	A function $f:U^*\to\R$ is \emph{$k$-stable on dataset $D$} if adding
	or removing any $k$ elements from $D$ does not change the value of
	$f$, that is, $f(D)=f(D')$ for all $D'$ such that
	$|D\triangle D'|\leq k$. We say $f$ is \emph{stable} on $D$ if it is
	(at least) 1-stable on $D$, and \emph{unstable} otherwise. 
	\label{def:kStab}
\end{defn}
The \emph{distance to instability} of a dataset $D\in U^*$ with
	respect to a
	function $f$ is the number of elements that must be added to or
	removed from $D$ to reach a dataset that is not stable. Note that
	$D$ is $k$-stable if and only if its distance to instability is at
	least $k$.

\begin{thm}[Privacy guarantee for $\A_{\sf stab}$]
	If the threshold $\thr=\log(1/\delta)/\epsilon$, and the distance to instability function $dist_f(D)=\arg\max\limits_{k}\left[f(D) \text{ is } k\text{-stable}\right]$, then Algorithm \ref{Alg:kStab} is $(\epsilon,\delta)$-differentially private.
	\label{thm:kstabPriv}
\end{thm}

\begin{proof}
We prove the above theorem by considering the two possibilities for any $D'$ s.t. $|D\Delta D'|= 1$: either $f(D)=f(D')$,  or $f(D)\neq f(D')$. We prove the privacy in these two cases via Lemmas \ref{lem:stabP1} and \ref{lem:stabP2}.
	
	\begin{lem}
		Let $D\in U^*$ be any fixed dataset. Assume that for any dataset $D'\in U^*$ s.t. $|D\Delta D'|= 1$, we have $f(D)=f(D')$. Then, for any output $s \in \cR \cup \{\bot\}$, we have: 
		$\Pr[\A_{\sf stab}(D,f)=s]\leq e^\epsilon \Pr[\A_{\sf stab}(D',f)=s].$
		\label{lem:stabP1}
	\end{lem}

	\begin{proof}
		First, note that with the instantiation in Theorem \ref{thm:kstabPriv}, the function $\dist_f$ has a global sensitivity of one. Therefore, by the guarantees of the Laplace mechanism (Lemma~\ref{lem:lap}), $\wds$ satisfies $\epsilon$-differential privacy. Since the set of possible outputs is the same (i.e., $\{f(D),\bot\}$)  for both $D$ and $D'$, and the decision to output $f(D)$ versus $\bot$ depends only on $\wds$, we get the statement of the lemma by the post-processing property of differential privacy (Lemma~\ref{lem:post}).
	\end{proof}
	
	\begin{lem}
		Let $D\in U^*$ be any fixed dataset. Assume that for any dataset $D'\in U^*$ s.t. $|D\Delta D'|= 1$, we have $f(D)\neq(D')$. Then, for any output $s\in\cR\cup\{\bot\}$, we have the following with probability at least $1-\delta$:
		$\A_{\sf stab}(D,f)=\A_{\sf stab}(D',f)=\bot$.
		\label{lem:stabP2}
	\end{lem}
	\begin{proof}
		Since $f(D)\neq f(D')$, it follows that $f(D)$ and $f(D')$ are unstable, i.e., $\dist_f(D)=\dist_f(D')=0$. This implies $$\Pr[\A_{\sf stab}(D,f)=\bot] = \Pr[\A_{\sf stab}(D',f)=\bot] = \Pr\left[{\sf Lap}\left(\frac{1}{\epsilon}\right) \leq \frac{\log(1/\delta)}{\epsilon}\right].$$ Since the density function for the Laplace distribution ${\sf Lap}(\lambda)$ is $\mu(x)=\frac{1}{2\lambda}e^{-|x|/\lambda}$, it follows that $\Pr\left[{\sf Lap}\left(\frac{1}{\epsilon}\right)\leq\frac{\log(1/\delta)}{\epsilon}\right]\geq 1-\delta$.
	\end{proof}
    
We get the statement of Theorem \ref{thm:kstabPriv} by
combining Lemmas \ref{lem:stabP1} and \ref{lem:stabP2}.
\end{proof}

\begin{thm}[Utility guarantee for $\A_{\sf stab}$ (\cite{ST13})]
	If the threshold $\thr=\log(1/\delta)/\epsilon$, the distance to instability function is chosen as in Theorem \ref{thm:kstabPriv}, and $f(D)$ is $\left(\left(\log(1/\delta)+\log(1/\beta)\right)/\epsilon\right)$-stable, then Algorithm \ref{Alg:kStab} outputs $f(D)$ with probability at least $1-\beta$.
	\label{thm:kstabUtil}
\end{thm}

\begin{algorithm}
	\caption{$\A_{\sf stab}$ \citep{ST13}: Private estimator for $f$ via distance to instability}
	\begin{algorithmic}[1]
		\REQUIRE dataset: $D$, function $f:U^*\to\R$, distance to instability $\dist_f:U^*\to\mathbb{R}$, threshold: $\thr$, privacy parameter $\epsilon>0$
		\STATE $\wds\leftarrow\dist_f(D)+{\sf Lap}\left(1/\epsilon\right)$
		\STATE {\bf If} $\wds > \thr$, {\bf then} output $f(D)$, {\bf else} output $\bot$ 
	\end{algorithmic}
	\label{Alg:kStab}
\end{algorithm}

\begin{algorithm}
	\caption{$\A_{\sf OQR}$: Online Query Release via distance to instability}
	\begin{algorithmic}[1]
		\REQUIRE dataset: $D$, query set $\cF=\{f_1,\cdots,f_m\}$ chosen online, distance to instability $\dist_{f_i}:U^*\to\mathbb{R}, \forall i\in[m]$, unstable query cutoff: $T$, privacy parameters $\epsilon,\delta >0$
		\STATE $c\leftarrow 0$, $\lambda\leftarrow \sqrt{32T\log(2/\delta)}/\epsilon$, $w\leftarrow 2\lambda\cdot{\log(2m/\delta)}$, and  $\widehat w\leftarrow w+{\sf Lap}(\lambda)$. 
		\FOR{$f\in \cF$ and $c\leq T$}
        	\STATE $\out\leftarrow \A_{\sf stab}\left(D, f, \dist_f,\Gamma=\widehat{w}, \epsilon=1/2\lambda\right)$
			\STATE {\bf If} $\out = \bot$, {\bf then} $c\leftarrow c + 1$  and $\widehat w\leftarrow w+{\sf Lap}(\lambda)$	
            \STATE Output $\out$
		\ENDFOR
		\end{algorithmic}
	\label{Alg:seqCompPriv}
\end{algorithm}

\subsection{Online Query Release via Distance to Instability Framework}
\label{sec:compStD}

Using Algorithm $\A_{\sf OQR}$ (Algorithm~\ref{Alg:seqCompPriv}), we show that for a set of $m$ queries $\mathcal{F}=\{f_1,\cdots,f_m\}$ to be answered on a dataset $D$, one can \emph{exactly} answer all but $T$ of them while satisfying differential privacy, as long as at most $T$ queries in $\mathcal{F}$ are not $\alpha$-stable, where $\alpha \approx \log(m)\sqrt{T}/\epsilon$. Notice that the dependence of $\alpha$ on the total number of queries ($m$) is logarithmic. In contrast, one would achieve a dependence of roughly $\sqrt{m}$ by using the advanced composition property of differential privacy \cite{DRV10}. 

The main design focus in this section is that the algorithms should be able to handle very generic query classes $\mathcal{F}$ under minimal assumptions. A salient feature of Algorithm $\A_{\sf OQR}$ is that it only requires the range $\R_i$ of the function $f_i:U^*\to\R_i$, where $f_i \in \mathcal{F}$, to be discrete for all $i \in [m]$.  

We provide the privacy and utility guarantees for Algorithm \ref{Alg:seqCompPriv} in Theorem~\ref{thm:cPriv3} and Corollary~\ref{cor:spAcc1}, respectively.
Surprisingly, the utility guarantee of $\A_{\sf OQR}$ has \emph{no dependence} on the cardinality of the set $\R_i$, for all $i \in [m]$.

\begin{thm}[Privacy guarantee for $\A_{\sf OQR}$]
	If for all functions $f\in\mathcal{F}$, the distance to instability function is $dist_f(D)=\argmax\limits_{k}\left[f(D) \text{ is } k\text{-stable}\right]$, then Algorithm \ref{Alg:seqCompPriv} is $(\epsilon,\delta)$-differentially private.
	\label{thm:cPriv3}
\end{thm}

\begin{proof}
In our proof, we use ideas from the proof of Theorem \ref{thm:kstabPriv} and the sparse vector technique (See Section \ref{sec:background} for a background on the technique). For clarity, we split the computation in Algorithm \ref{Alg:seqCompPriv} into two logical phases: First, for every query $f\in\cF$, $\A_{\sf OQR}$ either commits to $\top$,  or outputs $\bot$ based on the input dataset $D$. Next, if it commits to $\top$, then it  outputs $f(D)$. 

Now, let us consider two fictitious algorithms $\A_1$ and $\A_2$, where $\A_1$ outputs the sequence of $\top$ and $\bot$ corresponding to the first phase above, and $\A_2$ is invoked to output $f_i(D)$ only for the queries $f_i$ that $\A_1$ output $\top$. Notice that the combination of $\A_1$ and $\A_2$ is equivalent to $\A_{\sf OQR}$. Since $\A_1$ is essentially executing the sparse vector technique (Algorithm \ref{Alg:sparseVec}), by Theorem \ref{thm:spPriv}, it satisfies $(\epsilon,\delta/2)$-differential privacy. Next, we analyze the privacy for Algorithm $\A_2$.

Consider any particular query $f\in\cF$. For any dataset $D'$ s.t. $|D\Delta D'| = 1$, there are two possibilities: either $f(D)=f(D')$, or $f(D)\neq f(D')$.
When $f(D)=f(D')$, if $\A_1$ outputs $\bot$, algorithm $\A_2$ is not invoked and hence the privacy guarantee isn't affected. Moreover,if $\A_1$ outputs $\top$, we get the following lemma by the post-processing property of differential privacy (Lemma~\ref{lem:post}):

\begin{lem}
	Let $D\in U^*$ be any fixed dataset. Assume that for any dataset $D'\in U^*$ s.t. $|D\Delta D'| = 1$, we have $f(D)=f(D')$. Then, for any output $s\in\cR$, we have the following for the invocation of Algorithm $\A_2$: 
	$\Pr[\A_2(D,f)=s]=\Pr[\A_{2}(D',f)=s].$
	\label{lem:stabP3}
\end{lem}

When $f(D)\neq f(D')$, by Lemma \ref{lem:stabP2}, $\A_1$ outputs $\bot$ with probability at least $1-\delta/2m$. Therefore, we get that:

\begin{lem}
	Let $D\in U^*$ be any fixed dataset. Assume that for any dataset $D'\in U^*$ s.t. $|D\Delta D'| = 1$, we have $f(D)\neq f(D')$.  Then, Algorithm $\A_2$ is never invoked to output $f(D)$ with probability at least $1-\delta/2m$.
	\label{lem:stabP4}
\end{lem}

Now, consider the sequence of queries $f_1,\cdots,f_m$. Let $\cF_1$ be the set of queries where, for every $f\in\cF_1$, we have $f(D)=f(D')$. Let $\cF_2=\cF \backslash\cF_1$. Since Algorithm $\A_1$ is $(\epsilon,\delta/2)$-differentially private for all queries in $\cF$, it is also $(\epsilon,\delta/2)$-differentially private for all queries in $\cF_1$. Now since $|\cF_2|\leq m$, using Lemma \ref{lem:stabP4} and taking an union bound over all the queries in $|\cF_2|$, Algorithm $\A_2$ is never invoked for queries in $\cF_2$ with probability at least $1-\delta/2$. By general composition \cite{DR14}, this implies $(\epsilon,\delta)$-differential privacy for the overall algorithm $\A_{\sf OQR}$.
\end{proof}

\begin{cor}[Utility guarantee for $\A_{\sf OQR}$]
For any set of $m$ adaptively chosen queries $\mathcal{F}=\{f_1,\cdots,f_m\}$, let $dist_{f_i}(D)=\arg\max\limits_{k}\left[f_i(D) \text{ is } k\text{-stable}\right]$ for each $f_i$. Also, define $L(\alpha)=\left\{i : \ds_{f_i}(D) < \alpha \right\}$ for $\alpha=32\cdot\log\left(4mT/\min\left(\delta, \beta\right)\right)\sqrt{2T\log(2/\delta)}/\epsilon$    
. If $|L(\alpha)|\leq T$, then we have the following w.p. at least $1-\beta$: $\forall i\not\in L(\alpha)$, Algorithm $\A_{\sf OQR}$ (Algorithm \ref{Alg:seqCompPriv}) outputs $f_i(D)$. 
	\label{cor:spAcc1}
\end{cor}

\begin{proof}

The proof follows directly from Theorem \ref{thm:spAcc}.  To see this, note that Algorithm $\A_{\sf OQR}$ follows the same lines of Algorithm $\A_{\sf sparseVec}$ with slight adjustments. In particular, $\top$ in $\A_{\sf sparseVec}$ is replaced with $f(D)$ in $\A_{\sf OQR}$;~ $\delta$ in the setting of $\lambda$ in $\A_{\sf sparseVec}$ is replaced with $\delta/2$ in the setting of $\lambda$ in $\A_{\sf OQR}$;~ $w$ which is left arbitrary in $\A_{\sf sparseVec}$ is set to $2\lambda\,\log(2m/\delta)$ in $\A_{\sf OQR}$;~ $q$ in $\A_{\sf sparseVec}$ is replaced with $\dist_f(D)$ in $\A_{\sf stab}$;~ and $\widehat{q}$ in $\A_{\sf sparseVec}$ is replaced with $\wds$ in $\A_{\sf stab}$. Putting these together with Theorem~\ref{thm:spAcc} and the premise in the corollary statement (i.e., $\lvert\{i: \dist_{f_i}(D)<\alpha\}\rvert\leq T$) immediately proves the corollary with the specified value of $\alpha$. Note that by comparing Theorem~\ref{thm:spAcc} with the premise in the corollary, we can see that the value of $\alpha$ in the corollary is obtained by adding the value of $w$ as set in $\A_{\sf OQR}$ and the value of $\alpha$ as set in Theorem~\ref{thm:spAcc}.

\end{proof}

\subsection{Instantiation: Online Query Release via the Sub-sample and Aggregate Framework}
\label{sec:subSampleAgg}

While Algorithm $\A_{\sf OQR}$ has the desired property in terms of generality, it falls short in two critical aspects: i) it relies directly on the distance to instability framework (Algorithm $\A_{\sf stab}$ in Section \ref{sec:distInstab}) which does not provide an efficient way to compute the distance to instability for a given function, and ii) given a function class $\mathcal{F}$, it is unclear which functions from $\mathcal{F}$ satisfy the desired property of $\alpha$-stability. 
\begin{algorithm}
	\caption{$\A_{\sf subSamp}$: Online Query Release via sub-sample and aggregate}
	\begin{algorithmic}[1]
		\REQUIRE dataset: $D$, query set $\cF=\{f_1,\cdots,f_m\}$ chosen online, range of the queries: $\{\cR_1,\cdots,\cR_m\}$, unstable query cutoff: $T$, privacy parameters $\epsilon,\delta >0$, failure probability: $\beta$
		\STATE $k\leftarrow 136\cdot\log\left(4mT/\min\left(\delta, \beta/2\right)\right)\sqrt{T\log(2/\delta)}/\epsilon$ 
		\STATE Arbitrarily split $D$ into $k$ non-overlapping chunks of size $n/k$. Call them $D_1,\cdots,D_k$
        \FOR{$i\in[m]$}
        	\STATE Let $\cS_i=\{f_i(D_1),\cdots,f_i(D_k)\}$, and for every $r\in\cR_i$, let ${\sf ct}(r)=\#$ times $r$ appears in $\cS_i$ 
            {\STATE $\widehat{f}_i(D)\leftarrow\arg\max\limits_{r\in\R_i}\left[{\sf ct}(r)\right]$, $\dist_{\widehat{f}_i} \leftarrow\max\left\{0, \left(\max\limits_{r\in\cR_i}\left[{\sf ct}(r)\right]-\max\limits_{r\in\cR_i\backslash\widehat{f}_i(D)}\left[{\sf ct}(r)\right]\right)-1\right\}$\label{ln:abc}}
        \ENDFOR 
        \STATE Output $\A_{\sf OQR}\left(D,\left\{\widehat{f}_1,\cdots,\widehat{f}_m\right\}, \left\{\dist_{\widehat{f}_1},\cdots,\dist_{\widehat{f}_m}\right\}, T, \epsilon, \delta\right)$
	\end{algorithmic}
	\label{Alg:subSamp}
\end{algorithm}

 In Algorithm $\A_{\sf subSamp}$ (Algorithm \ref{Alg:subSamp}), we address both of these concerns by instantiating the distance to instability function in Algorithm $\A_{\sf OQR}$ with the sub-sample and aggregate framework (as done in \cite{ST13}). We provide the privacy and accuracy guarantees for $\A_{\sf subSamp}$ in Corollary~\ref{cor:privSubSamp}, and Theorem~\ref{thm:utilitySubAgg}, respectively.
In Section \ref{sec:privClass}, we show how Algorithm $\A_{\sf subSamp}$ can be used for classification problems without relying too much on the underlying learning model (e.g., convex versus non-convex models).

The key idea in $\A_{\sf subSamp}$ is as follows: i) First, arbitrarily split the dataset $D$ into $k$ sub-samples of equal size, $D_1,\cdots,D_k$, ii) For each query $f_i \in \mathcal{F}$, where $i \in [m]$, and each $r\in\R_i$, compute ${\sf ct}(r)$, which is the number of sub-samples $D_j$, where $j \in [k]$, for which $f_i(D_j) = r$, iii) Define $\widehat f_i(D)$ to be the $r\in\R_i$ with the largest ${\sf ct}$, and the distance to instability function ${\sf dist}_{\widehat f_i}$ to correspond to the the difference between the largest ${\sf ct}$ and the second largest ${\sf ct}$ among all $r\in\R_i$, iv) Invoke $\A_{\sf OQR}$ with $\widehat f_i$ and ${\sf dist}_{\widehat f_i}$. Now, note that ${\sf dist}_{\widehat f_i}$ is \emph{always} efficiently computable. Furthermore, Theorem \ref{thm:utilitySubAgg} shows that if $D$ is a dataset of $n$ i.i.d. samples drawn from some distribution $\mathcal{D}$, and $f_i$ on a dataset of $n/k$ i.i.d. samples drawn from $\mathcal{D}$ matches some $r\in\R_i$ w.p. at least $3/4$, then with high probability $\widehat f_i(D)$ is a stable query. 


\begin{cor}[Privacy guarantee for $\A_{\sf subSamp}$]
	Algorithm \ref{Alg:subSamp} is $(\epsilon,\delta)$-differentially private.
	\label{cor:privSubSamp}
\end{cor}

The proof of Corollary \ref{cor:privSubSamp} follows immediately from the privacy guarantee for $\A_{\sf OQR}$ (Algorithm~\ref{Alg:seqCompPriv}).

\begin{thm}[Utility guarantee for $\A_{\sf subSamp}$]
     Let $\cF$ denote any set of $m$ adaptively chosen queries, and $D$ be a dataset of $n$ samples drawn i.i.d. from a fixed distribution $\mathcal{D}$. For $k = 136\cdot\log\left(4mT/\min\left(\delta, \beta/2\right)\right) \cdot \sqrt{T\log(2/\delta)}/\epsilon$, let $\bar{L}\subseteq{\mathcal{F}}$ be a set of queries s.t. for every  $f\in\bar{L}$, there exists some $x_f$ for which $f(\widehat D)=x_f$ w.p. at least $3/4$ over drawing a dataset $\widehat D$ of $n/k$  i.i.d. data samples from $\mathcal{D}$. If $|\bar{L}|\geq m-T$, then w.p. at least $1-\beta$ over the randomness of Algorithm $\A_{\sf subSamp}$ (Algorithm~\ref{Alg:subSamp}), we have the following: $\forall f\in{\bar L}$, Algorithm $\A_{\sf subSamp}$ outputs $x_f$. Here, $(\epsilon,\delta)$ are the privacy parameters.
   \label{thm:utilitySubAgg}
\end{thm}

\begin{proof}
For a given query $f\in \mathcal{F}$, let $X^{(i)}_f$ be the random variable that equals to one if $f(D_i)$ in Algorithm  $\A_{\sf subSamp}$ equals $x_f$, and zero otherwise. Thus, we have $\Pr[X^{(i)}_f=1]\geq 3/4$ by assumption. By the standard Chernoff-Hoeffding bound, we get $\sum\limits_{i=1}^k X^{(i)}_f\geq 3k/4-\sqrt{k\log(2m/\beta)/2}$ with probability at least $1-\beta/2m$. If $k\geq 72\log(2m/\beta)$, then the previous expression is at least $2k/3$. By the union bound, this implies that with probability at least $1-\beta/2$, we have $\dist_{\widehat{f}}\geq k/3$ for every $f\in{\bar L}$. Furthermore, to satisfy the distance to instability condition in Corollary \ref{cor:spAcc1}, we need $k/3\geq 32\cdot\log\left(4mT/\min\left(\delta, \beta/2\right)\right)\sqrt{2T\log(2/\delta)}/\epsilon$. Both the conditions on $k$ are satisfied by setting $k=136\cdot\log\left(4mT/\min\left(\delta, \beta/2\right)\right)\sqrt{T\log(2/\delta)}/\epsilon$. Using Corollary \ref{cor:spAcc1} along with this value of $k$, we get the statement of the theorem.
\end{proof}

    \section{Private Binary Classification}
\label{sec:privClass}


In this section, we consider the problem of privately answering binary classification queries (in the standard agnostic PAC model), where for each domain point $x \in \cX$, a classifier $h$ assigns a hard label to $x$, i.e., $h(x)\in \{0,1\}$. 

In Section~\ref{sec:privClass-hard}, we show how to use the sub-sample and aggregate framework combined with the sparse-vector technique (Algorithm~\ref{Alg:subSamp}) to construct generic private algorithm for answering binary classification queries on public data based on \emph{private} training data. 

Next, in Section~\ref{sec:label-priv}, we discuss the implications of our construction on label-private learning (and equivalently, on private semi-supervised knowledge transfer) in the standard (agnostic) PAC model. In particular, we show that the set of labels privately generated by our algorithm in Section~\ref{sec:privClass-hard} based on an input private training set, can be used to construct a new training set. The new set can then be used to train a (non-private) learning algorithm to finally output an accurate classifier that can be used to answer indefinite number of classification queries. The whole procedure would remain $(\epsilon, \delta)$ differentially private with respect to the original dataset. This is because the second training step is a post-processing of the output labels generated by our differentially private algorithm. We prove sample complexity upper bounds on the size of the original dataset that would guarantee any desired accuracy and confidence for the final output classifier.

We will use $\cX$ to denote an abstract data domain (e.g., the space of feature vectors). We let $\cY=\{0,1\}$, i.e., the set of binary labels. A training set, denoted by $D$, is a set of $n$ private binary-labeled data points $\{(x_1, y_1), \dots, (x_n, y_n)\} \subseteq \cX\times \cY$ drawn i.i.d. from some (arbitrary unknown) distribution $\cD$ over $\cX\times\cY$. Sometimes, we will refer to the induced marginal distribution over $\cX$ as $\cD_{\cX}$.




\subsection{Privately answering binary classification queries}\label{sec:privClass-hard}

In this section, we instantiate our sub-sample and aggregate framework ($\cA_{\sf subSamp}$ from Section~\ref{sec:subSampleAgg}) with the binary classification setting. We set the dataset $D$ to be $n$ i.i.d. samples from the distribution $\cD$. We also construct a set of classification queries $\cQ=\{x_1,\cdots,x_m\}\subseteq \cX$, and let $\{y_1,\cdots,y_m\}\subseteq \{0,1\}^m$ be the corresponding labels which are hidden. Corresponding to the query class $\cF=\{f_1,\cdots,f_m\}$ in Algorithm $\cA_{\sf subSamp}$, we define $f_i:\cH\to\{0,1\}$ to be $f_i(h)=h(x_i)$ for a given hypothesis $h\in\cH$. With a slight abuse of notation, we will refer $f_i(D_j)$ in Algorithm $\cA_{\sf subSamp}$ as $f_i(\Theta(D_j))$, where $\Theta$ is an agnostic PAC learner (see  Definitions~\ref{def:agPAC} and \ref{defn:learner}). 

\begin{cor}
Let $\cF$ and $\cQ$ denote the query classes as defined above, and $D$ be a dataset of $n$ i.i.d. examples from a distribution $\cD$. Let $\alpha, \beta \in (0, 1)$. Let $\gamma \triangleq \min\limits_{h\in\cH}\err(h; \cD)$ (as in Definition \ref{def:agPAC}). In Algorithm \ref{Alg:subSamp} (Algorithm $\cA_{\sf subSamp}$), suppose we set the unstable query cutoff as $T=3\left((\gamma+\alpha)m+\sqrt{m\log(m/\beta)/2}\right)$ (where $k$ is defined in Algorithm $\cA_{\sf subSamp}$). If $n, \alpha,$ and $\beta$ are such that $\Theta$ is $(\alpha, \beta/k, n/k)$-agnostic PAC learner (Definition~\ref{defn:learner}), then i) with probability at least $1-2\beta$, Algorithm $\cA_{\sf subSamp}$ does not halt before answering all the $m$ queries in $\cF$, and outputs $\bot$ for at most $T$ queries in $\cF$; and ii) the misclassification rate of Algorithm $\cA_{\sf subSamp}$ is at most $T/m$.
\label{cor:privPAC}
\end{cor}


\begin{proof}
First, notice that $\Theta$ is an $(\alpha, \beta/k, n/k)$-agnostic PAC learner, hence w.p. $\geq 1-\beta$, the misclassification rate of $\Theta(D_j)$ for all $j\in [k]$ is at most $\gamma +\alpha$. So, by the standard Hoeffding bound, with probability at least $1-\beta$ none of the $\Theta(D_j)$'s misclassify more than $(\gamma+\alpha)m+\sqrt{m\log(k/\beta)/2}$ \mbox{$< (\gamma+\alpha)m+\sqrt{m\log(m/\beta)/2}\triangleq B$} queries in $\cQ$. Now, we use the following lemma to get a bound on the number of queries for which at least $k/3$ sub-samples result in a misclassification.

\begin{lem}
\label{lem:count}
Consider a set of $\{(x_1, y_1), \ldots, (x_m, y_m)\}\subset \cX\times \cY$, and $k$ binary classifiers $h_1, \ldots, h_k$, where each classifier is guaranteed to make at most $B$ mistakes in predicting the $m$ labels $\{y_1, \ldots, y_m\}$. Then, for any $\xi \in (0,1/2]$, 
$$\bigg\lvert\bigg\{i\in [m]: \lvert \{j\in [k]: h_j(x_i)\neq y_i \}\rvert  > \xi k\bigg\}\bigg\rvert < B/\xi$$
\end{lem}

Therefore, by the counting argument from Lemma~\ref{lem:count}, there are at most $3B$ queries $f\in\cF$ such that the set $S=\{f(\Theta(D_1)),\cdots,f(\Theta(D_k))\}$ has number of ones (or, zeros) that is $>k/3$. Now, part 1 of the corollary follows by the same proof technique that has been used to prove Theorem \ref{thm:utilitySubAgg}. Moreover, by the same lemma, w.p. $1-\beta$, in each of the remaining $m-3B$ queries, there are at least $2k/3$ classifiers that output the correct label. Hence, again by the same proof technique of Theorem~\ref{thm:utilitySubAgg}, w.p. $\geq 1-2\beta$ our instantiation of Algorithm~\ref{Alg:subSamp} will correctly classify such queries. Hence, w.p. $\geq 1-2\beta,$ the misclassification rate is $3B/m =T/m$. This completes the proof.
\end{proof}

\begin{remark}
We can obtain similar guarantees for multi-class classification with an almost identical proof as of Corollary \ref{cor:privPAC}. We provide the guarantees for binary classification in Corollary \ref{cor:privPAC} for simplicity.
\end{remark}  

\paragraph{Explicit misclassification rate:} 
For VC classes, we now show how to obtain explicit misclassification rates for our algorithm in terms of the VC-dimension of the hypothesis class. Let $V$ denote the  VC-dimension of the hypothesis class $\cH$. Then by standard uniform convergence arguments (\cite{shalev2014understanding}), one can show that there exists an $(\alpha, \beta, n/k)$-agnostic PAC learner with $\alpha=\tilde{O}\left(\sqrt{{kV}/{n}}\right)$, and hence it has a misclassification rate of $\approx\gamma+ \tilde{O}\left(\sqrt{{kV}/{n}}\right)$ when trained on a dataset of size $n/k$. Assuming $m= \tilde{\Omega}(1/\alpha^2)$, the setting of $T$ in Corollary \ref{cor:privPAC} becomes $T=O(m(\alpha+\gamma))$. Hence, setting $k$ as in Algorithm~\ref{Alg:subSamp} implies that $k\approx \tilde{O}\left(m^{2/3}\,V^{1/3}/n^{1/3}\right)$. To see this, note that $T$ depends on $\alpha$ which itself depends on $k$, and also note that $k$ depends on $T$. Hence, by putting these together and solving for $k$, we reach the above expression of $k$. Therefore, Corollary~\ref{cor:privPAC} implies that our algorithm yields a misclassification rate of $\approx \gamma + \tilde{O}\left(m^{1/3}V^{2/3}/n^{2/3}\right)$ (where $m=\tilde{\Omega}\left(1/\alpha^2\right)=\tilde{\Omega}\left((n / V)^{4/5}\right)$). This implies that we can answer up to $\tilde\Omega\left((n / V)^{4/5}\right)$ queries with misclassification rate $\tilde{O}\left(V^{2/5}/n^{2/5}\right)$, which is only a factor of $\approx (n/V)^{1/10}$ worse than the optimal non-private misclassification rate. 

\noindent In the realizable case when $\gamma=0$, the misclassification rate of the PAC learner is $\alpha=\tilde{O}\left({kV}/{n}\right)$. Assuming $m= \tilde{\Omega}(1/\alpha^2)$, then the setting of $T$ in Corollary \ref{cor:privPAC} becomes $T=O(m\alpha)$. In this case, setting $k$ as in Algorithm~\ref{Alg:subSamp} implies that $k\approx \tilde{O}\left(mV/n\right)$. Hence, Corollary~\ref{cor:privPAC} implies that our algorithm yields a misclassification rate of $\approx \tilde{O}\left(mV^2/n^2\right)$ (where $m = \tilde{\Omega}\left(1/\alpha^2\right)=\tilde{\Omega}\left(\left(n / V\right)^{4/3}\right)$). Again, this implies that in the realizable setting, we can answer up to $\tilde{\Omega}\left(\left(n / V\right)^{4/3}\right)$ queries with with misclassification rate $\tilde{O}\left(V^{2/3}/n^{2/3}\right)$, which is a factor of $\approx (n/V)^{1/3}$ worse than the optimal non-private misclassification rate. 


We formally state these conclusions in the following theorems. 

\begin{thm}[Private classification queries in the PAC setting]\label{thm:class-quer-pac}
Consider a hypothesis class $\cH$ of VC-dimension $V$. Consider the realizable setting where an unknown labeling function in $\cH$ generates the labels. Let $\alpha, \beta \in (0,1)$. Let $n$ be the size of the private training set, and \mbox{$m\geq 4\log\left(1/\alpha\beta\right)/\alpha^2$} be the maximum number of binary classification queries (public feature vectors) allowed\footnote{Note that $m$ only represents an upper on the number of queries our algorithm can answer. The number of queries actually submitted can be less than $m$.}. In the above instantiation of Algorithm \ref{Alg:subSamp}, set $T=3\left(\alpha m+\sqrt{m\log(m/\beta)/2}\right)$ and $k$ as in Algorithm \ref{Alg:subSamp} (i.e., $k=136\,\log\left(4mT/\min\left(\delta, \beta\right)\right)\sqrt{T\log(2/\delta)}/\epsilon$). If 
$$n=O\left(\frac{\left(V\log(1/\alpha)+\log(m/\beta)\right)\log\left(mT/\min\left(\delta, \beta\right)\right)\sqrt{T\log(1/\delta)}/\epsilon}{\alpha}\right)=\tilde{O}\left(V\sqrt{m/\alpha}\right),$$
then w.p. $\geq 1-\beta,$ the instantiation of Algorithm~\ref{Alg:subSamp} classifies at least $m-T$ queries correctly. 

\end{thm}

\begin{thm}[Private classification queries in the agnostic PAC setting]\label{thm:class-quer-agpac}
Consider a hypothesis class $\cH$ of VC-dimension $V$. Let $\gamma=\min\limits_{h\in\cH}\err(h; \cD)$ (as in Definition~\ref{def:agPAC}). Let $\alpha, \beta \in (0,1)$. Let $n$ be the size of the private training set, and \mbox{$m\geq 4\log\left(1/\alpha\beta\right)/\alpha^2$} be the maximum number of binary classification queries allowed. In the above instantiation of Algorithm \ref{Alg:subSamp}, set \mbox{$T=3\left(\left(\alpha+\gamma\right) m+\sqrt{m\log(m/\beta)/2}\right)$} and $k$ as in Algorithm \ref{Alg:subSamp}. If 
$$n=O\left(\frac{\left(V+\log(m/\beta)\right)\log\left(mT/\min\left(\delta, \beta\right)\right)\sqrt{T\log(1/\delta)}/\epsilon}{\alpha^2}\right)=\tilde{O}\left(V\sqrt{m}/\alpha^{3/2}\right),$$
then w.p. $\geq 1-\beta,$ the instantiation of Algorithm~\ref{Alg:subSamp} classifies at least $m-T$ queries correctly. 

\end{thm}

\subsection{Sample complexity bounds for efficient label-private learning}\label{sec:label-priv}

In this section, we build on our results above to achieve a stronger objective than merely answering classification queries. In particular, we show how to bootstrap from our construction above and produce a private learner that publishes an accurate privatized classifier, which can then be used to label as many feature vectors as we desire. The idea is fairly simple: we use our private construction above to generate labels on a sufficiently large set of unlabeled domain points. Then, we use the resulting labeled set as a new training set for any standard (non-private) learner, which in turn outputs an accurate private classifier. 

Our final construction can be viewed as a private learner in either of the following situations: (i) the setting where the original training set is private but we still have access to public unlabeled data, or (ii) the setting where only the labels of the training set are considered private and we do not have access to public unlabeled data. Note that the second setting can be reduced to the first by splitting the training set into two parts, and throwing away the labels of one of them.
We formalize our idea below, and prove explicit sample complexity bounds for the final private learner in both PAC and agnostic PAC settings. 

Let $h_j=\Theta(D_j),~ j\in [k]$ be the $k$ classifiers generated by the non-private learner in our instantiation of Algorithm~\ref{Alg:subSamp} described in Section~\ref{sec:privClass-hard}. For a target accuracy and confidence parameters $\alpha, \beta \in (0, 1)$, we will always assume that $\Theta$ is $(\alpha, \beta/k, n/k)$-(agnostic) PAC learner for the hypothesis class $\cH$ (for appropriately chosen $n$). 

Let $\hpv$ denote the mapping defined by our instantiation of Algorithm~\ref{Alg:subSamp} on a single input feature vector (query). That is, for $x\in\cX$, $\hpv(x)\in\{0, 1, \bot\}$ denotes the output of our algorithm on a single input query $x$. Note that without loss of generality, we can view $\hpv$ as a binary classifier. In particular, as far as our accuracy guarantees are concerned, we may replace an output $\bot$ with a uniformly random label in $\{0,1\}$. Our private learner is described in Algorithm~\ref{Alg:ppac} below. 

\begin{algorithm}
	\caption{$\A_{\ppac}$: Private Learner}
	\begin{algorithmic}[1]
		\REQUIRE Unlabeled set of $m$ i.i.d. feature vectors: $\cS=\{x_1, \ldots, x_m\}$, oracle access to our private classifier $\hpv$, oracle access to (agnostic) PAC learner $\Theta$.
		\FOR{$t=1, \ldots, m$}
        \STATE $\hy_t\leftarrow \hpv(x_t)$
        \ENDFOR
        \STATE Let $\tD = \{(x_1, \hy_1), \ldots, (x_m, \hy_m)\}$. 
        \STATE \textbf{Output} $\hh\leftarrow \Theta(\tD).$
	\end{algorithmic}
	\label{Alg:ppac}
\end{algorithm}

We assume that $\hpv$ represents one invocation of our instantiation of Algorithm~\ref{Alg:subSamp} with number of queries set to $m$ (the size of $\cS$) and cutoff parameter $T$ set as in Corollary~\ref{cor:privPAC}. 

Note that $\A_{\ppac}$ is $(\epsilon, \delta)$-differentially private since it is a postprocessing of the labels $\hy_1, \ldots, \hy_m$ generated by our private $(\epsilon, \delta)$ differentially private algorithm described in Section~\ref{sec:privClass-hard}. 
We also note that $\A_{\ppac}$ is computationally efficient as long as the underlying non-private learner $\Theta$ is computationally efficient. 

Note that the mapping $\hpv$ is independent of the input feature vector (query) $x\in\cX$; it only depends on the input training set (in particular, on $h_1, \ldots, h_k$) and on the internal randomness (due to noise in in the threshold $\widehat{w}$ and the distance $\wds$).  

We now make the following claim about $\hpv$. In this claim, $m$ and $k$ refer to the parameters in the instantiation of Algorithm~\ref{Alg:subSamp}.

\begin{claim}\label{cl:ppac}
Let $0<\beta \leq \alpha <1$. Let $m\geq 4\log(1/\alpha\beta)/\alpha^2$. Suppose that $\Theta$ is $(\alpha, \beta/k, n/k)$-(agnostic) PAC learner for the hypothesis class $\cH$. Then, with probability at least $1-2\beta$ (over the randomness of the input training set $D$ and the internal randomness in Algorithm~\ref{Alg:subSamp}), we have $\err(\hpv; \cD)\leq 3\gamma + 7\alpha=O(\gamma+\alpha)$, where $\gamma=\min\limits_{h\in\cH}\err(h; \cD)$.
\end{claim}

\begin{proof}
The proof largely relies on the proof of Corollary~\ref{cor:privPAC}. First, note that w.p. $\geq 1-\beta$ (over the randomness of the input dataset $D$) for all $j\in [k],$ $\err(h_j; \cD)\leq \alpha$. We will thereafter condition on this event. 

Let $x_1, \ldots, x_m$ be a sequence of i.i.d. feature vectors, and $y_1, \ldots, y_m$ be the corresponding (unknown) labels. Now, for every $t\in [m],$ define 
$$v_t\triangleq \ind\left(\lvert\{j\in [k]: h_j(x_t)\neq y_t\}\rvert > k/3\right)$$
Note that since $(x_1, y_1), \ldots, (x_m, y_m)$ are i.i.d., it follows that $v_1, \ldots, v_m$ are i.i.d. (note that this is true since we conditioned on the dataset $D$). Note that in the proof of Corollary~\ref{cor:privPAC}, we showed that 
$$\pr{x_1,\ldots, x_m}{\frac{1}{m}\sum_{t=1}^m v_t > 3\left(\alpha+\gamma +\sqrt{\frac{\log(m/\beta)}{2m}}\right)} < \beta$$
Hence, for any $t\in [m],$ 
$$\ex{x_1,\ldots, x_m}{v_t}=\ex{x_1,\ldots, x_m}{\frac{1}{m}\sum_{t=1}^m v_t}<\beta + 3\left(\alpha+\gamma +\sqrt{\frac{\log(m/\beta)}{2m}}\right) \leq 7\alpha + 3\gamma$$
where the first equality follows from the fact that $v_1, \ldots, v_m$ are i.i.d., and the last inequality follows from the assumptions that $\beta\leq \alpha$ and that $m\geq 4\log(1/\alpha\beta)/\alpha^2.$ 

For every $t\in [m],$ let $\bar{v}_t= 1-v_t$, i.e., the negation of $v_t$. As in the proof of Corollary~\ref{cor:privPAC}, by invoking the same technique used before in the proof of Theorem~\ref{thm:utilitySubAgg}, then we can show that w.p. at least $1-\beta$ over the internal randomness (noise) in Algorithm~\ref{Alg:subSamp}, for all $t\in [m],$ we have 
$$\bar{v}_t=1 \Rightarrow \hpv(x_t)=y_t$$
Hence, conditioned on this event (over the internal randomness of Algorithm~\ref{Alg:subSamp}), for any $t\in [m],$ we have
$$\pr{x_t}{\hpv(x_t)\neq y_t}\leq \pr{x_t}{v_t=1} = \ex{x_t}{v_t}\leq 7\alpha+3\gamma.$$
Hence, putting all together, w.p. $\geq 1-2\beta,$ $\err(\hpv; \cD)\leq 7\alpha+3\gamma.$

\end{proof}

We are now ready to state and prove the main statements of this section. 

\begin{thm}[Sample complexity bound for Efficient Label-Private PAC Learning]\label{thm:ppac}

Let $\cH$ be a hypothesis class of VC-dimension $V$. Assume realizability holds, i.e., there is an unknown labeling function in $\cH$ generates the true labels. Let $0<\beta\leq \alpha < 1$. Let $m$ be the number of i.i.d. examples such that $\Theta$ is $(\alpha, \beta, m)$-agnostic PAC learner of $\cH$, namely, let $m=O\left(\frac{V+\log(1/\beta)}{\alpha^2}\right)$. Let the parameters $T$ and $k$ of the instantiation of Algorithm~\ref{Alg:subSamp} (in Section~\ref{sec:privClass-hard}) be set as in Theorem~\ref{thm:class-quer-pac} (with $m$ set as above). If the size of the input private dataset $n$ is as in Theorem~\ref{thm:class-quer-pac}, that is, 
$n=\tilde{O}\left(V^{3/2}/\alpha^{3/2}\right),$ then, w.p. $\geq 1-3\beta,$ the output hypothesis $\hh$ of $\A_{\ppac}$ (Algorithm~\ref{Alg:ppac}) satisfies  $\err(\hh; \cD)=O(\alpha)$.
\end{thm}

\begin{proof}
Let $h^*\in\cH$ denote the true labeling hypothesis. We will denote the true distribution $\cD$ over $\cX\times\cY$ as $(\cD_{\cX}, h^*)$, where $\cD_{\cX}$ denotes the marginal distribution over $\cX$ induced by $\cD$. The notation $(\cD_{\cX}, h^*)$ refers to the fact that the feature vectors are drawn from $\cD_{\cX}$, and the corresponding labels are generated by applying $h^*$ to the feature vectors. 

For the remainder of the proof, we condition on the event in the statement of Claim~\ref{cl:ppac}, i.e., we condition on the input dataset $D$ and the internal randomness of Algorithm~\ref{Alg:subSamp} such that $\err(\hpv; \cD)\leq 3\gamma+7\alpha$. Note such event occurs w.p. $\geq 1-2\beta$ over the dataset $D$ and the internal randomness of Algorithm~\ref{Alg:subSamp}. 

Let $\tD=\{(x_1, \hy_1), \ldots, (x_m, \hy_m)\}$ be the \emph{new} training set generated by $\A_{\ppac}$ (Algorithm~\ref{Alg:ppac}), where $m$ is set as in the theorem statement. Using the same style of notation as above, note that each $(x_t, \hy_t),~t\in [m],$ is drawn independently from $(\cD_{\cX}, \hpv)$. Now, since $\Theta$ is $(\alpha, \beta, m)$-agnostic PAC learner for $\cH$, w.p. $\geq 1-\beta$ (over $\tD$), the output hypothesis $\hh$ satisfies
$$\err(\hh; (\cD_{\cX}, \hpv))-\err(h^*; (\cD_{\cX}, \hpv))\leq \err(\hh; (\cD_{\cX}, \hpv))-\min\limits_{h\in \cH}\err(h; (\cD_{\cX}, \hpv))\leq \alpha$$

Observe that 
$$\err(h^*; (\cD_{\cX}, \hpv))=\ex{x\sim\cD_{\cX}}{\ind\left(h^*(x)\neq \hpv(x)\right)}=\err(\hpv; (\cD_{\cX}, h^*))=\err(\hpv; \cD)\leq 7\alpha$$
where the last inequality follows from Claim~\ref{cl:ppac} (where $\gamma=0$ due to realizability). 
Hence, we have $\err(\hh; (\cD_{\cX}, \hpv))\leq 8\alpha$. Furthermore, observe 
\begin{align*}
\err(\hh; \cD)=\ex{x\sim\cD_{\cX}}{\ind(\hh(x)\neq h^*(x))}&\leq \ex{x\sim\cD_{\cX}}{\ind(\hh(x)\neq \hpv(x)) +\ind(\hpv(x)\neq h^*(x))}\\
&=\err(\hh; (\cD_{\cX}, \hpv))+\err(\hpv; \cD)\\
&\leq 15\alpha
\end{align*}
Hence, we conclude that w.p. $\geq 1-3\beta$ (over the private dataset $D$, the internal randomness of Algorithm~\ref{Alg:subSamp}, and the set of public feature vectors $\cS$), we have $\err(\hh; \cD)\leq 15 \alpha$. 
\end{proof}

\begin{thm}[Sample complexity bound for Efficient Label-Private Agnostic PAC Learning]\label{thm:agppac}
Let $\cH$ be a hypothesis class of VC-dimension $V$. Let $0<\beta\leq \alpha < 1$. Let $m=O\left(\frac{V+\log(1/\beta)}{\alpha^2}\right)$. Let the parameters $T$ and $k$ of the instantiation of Algorithm~\ref{Alg:subSamp} (in Section~\ref{sec:privClass-hard}) be set as in Theorem~\ref{thm:class-quer-agpac} (with $m$ set as above). If $n$ is as in Theorem~\ref{thm:class-quer-agpac}, that is, 
$n=\tilde{O}\left(V^{3/2}/\alpha^{5/2}\right),$ then, w.p. $\geq 1-3\beta,$ the output hypothesis $\hh$ of $\A_{\ppac}$ (Algorithm~\ref{Alg:ppac}) satisfies  $\err(\hh; \cD)=O(\alpha+\gamma)\,$ (where $\gamma=\min\limits_{h\in\cH}\err(h; \cD)$).
\end{thm}

\begin{proof}
The proof follows similar lines to the proof of Theorem~\ref{thm:ppac}. First, let $\tlh\triangleq \arg\min\limits_{h\in\cH}\err(h; \cD)$. Note that it follows that $\err(\tlh; \cD)=\gamma$. As in the proof of Theorem~\ref{thm:ppac}, we start by conditioning on the event in the statement of Claim~\ref{cl:ppac}. 

Since $\Theta$ is $(\alpha, \beta, m)$-agnostic PAC learner for $\cH$, w.p. $\geq 1-\beta$ (over $\tD$), the output hypothesis $\hh$ satisfies
$$\err(\hh; (\cD_{\cX}, \hpv))-\err(\tlh; (\cD_{\cX}, \hpv))\leq \err(\hh; (\cD_{\cX}, \hpv))-\min\limits_{h\in \cH}\err(h; (\cD_{\cX}, \hpv))\leq \alpha$$ 
Next, note that 
\begin{align*}
\err(\tlh; (\cD_{\cX}, \hpv))=\ex{x\sim\cD_{\cX}}{\ind\left(\tlh(x)\neq \hpv(x)\right)}&\leq \ex{(x, y)\sim\cD}{\ind\left(\tlh(x)\neq y\right)+\ind\left(\hpv(x)\neq y\right)}\\
&\leq \err(\tlh; \cD) + \err(\hpv; \cD)\\
&\leq 4\gamma + 7\alpha
\end{align*}
where the last inequality follows from the definition of $\tlh$ and from Claim~\ref{cl:ppac}. Hence, we have $\err(\hh; (\cD_{\cX}, \hpv))\leq 4\gamma+8\alpha$. Moreover, observe that 
\begin{align*}
\err(\hh; \cD)=\ex{(x, y)\sim\cD}{\ind(\hh(x)\neq y}&\leq \ex{(x, y)\sim\cD}{\ind(\hh(x)\neq \hpv(x)) +\ind(\hpv(x)\neq y)}\\
&=\err(\hh; (\cD_{\cX}, \hpv))+\err(\hpv; \cD)\\
&\leq 7\gamma+15\alpha
\end{align*}
Hence, we conclude that w.p. $\geq 1-3\beta$ (over the private dataset $D$, the internal randomness of Algorithm~\ref{Alg:subSamp}, and the set of public feature vectors $\cS$), we have $\err(\hh; \cD)\leq 7\gamma+15 \alpha$. 

\end{proof}

\paragraph{Implications and comparison to prior work on label privacy:} Our results above have important implications on private learning in a setting where the learner is required to only protect the privacy of the labels, or in a setting where the learner is required to preserve the privacy of the entire sample but it has access to public unlabeled data from the same distribution. In both of these settings, the above theorems show that we can privately and \emph{efficiently} learn any given concept class. In particular, Theorem~\ref{thm:ppac} shows that in the realizable case our efficient construction $\A_{\ppac}$ is of a label-private PAC learner, and gives a sample complexity upper bound that is only a factor of $\tilde{O}\left(\sqrt{V/\alpha}\right)$ worse than the optimal non-private sample complexity. Theorem~\ref{thm:agppac} gives analogous guarantees in the agnostic (non-realizable) case; our sample complexity upper bound for the agnostic case is also only a factor of $\tilde{O}\left(\sqrt{V/\alpha}\right)$ larger than the optimal non-private sample complexity. However, we note that the accuracy of the final output classifier in the agnostic case may have a suboptimal dependency (by a small constant factor) on $\gamma\triangleq \min\limits_{h\in\cH}\err(h; \cD)$. 

\noindent Label-private learning has been considered before in \cite{chaudhuri2011sample} and \cite{BeimelNS16}. Both works have only considered pure, i.e., $(\epsilon, 0)$, differentially private learners for those settings, and the constructions in both works are computationally inefficient. In particular, the work of \cite{chaudhuri2011sample} gave upper and lower bounds on the sample complexity in terms of the doubling dimension. Their upper bound was given via an inefficient construction, and it also involves a smoothness condition on the distribution of the features $\cD_{\cX}$. The work of \cite{BeimelNS16} showed that the sample complexity (of pure differentially label-private learners) can be characterized in terms of the VC dimension. They proved an upper bound on the sample complexity via an inefficient construction, and only for the realizable case. The bound of \cite{BeimelNS16} is only a factor of $O(1/\alpha)$ worse than the optimal non-private bound for the realizable case. To the best of our knowledge, our constuction is the first efficient construction with non-trivial sample complexity upper bounds for both the realizable and agnostic settings. 
 
	\section{Privately Answering Soft-label classification Queries}\label{sec:privClass-soft}

To show the applicability of our algorithms to a broader range of problems, in this section, we will consider a more general setting for the classification problem. Namely, we consider the soft-label setting where the output of the learning algorithm is a mapping $h$ that provides a soft prediction in $[0, 1]$ for each domain point. In particular, $h(x)$ can be viewed as an estimate for the probability that the true label is $1$ conditioned on the feature vector being $x$, denoted as $p(1|x)$. 

We build on the generic algorithm of Section~\ref{sec:subSampleAgg} (Algorithm~\ref{Alg:subSamp}) to construct a private algorithm for answering soft-label classification queries. Our algorithm is conservative in its use of the privacy budget, that is, the error performance (or equivalently, the sample complexity requirement) scales only with the number of queries that are ``bad'' in some sense that will be precisely defined soon.

We start by describing a generic instantiation of our private algorithm. Our algorithm only requires a black-box access to any generic (non-private) learner that outputs a classifier based on private training data. 
We provide formal utility guarantees for our algorithm, and study some basic conditions on the underlying non-private learner under which stronger and sharper utility guarantees can be achieved. 

As before, we will use $\cX$ to denote an abstract data domain (e.g., the space of feature vectors). We let $\cY=\{0,1\}$, i.e., the set of binary labels , and $\bcY=[0, 1]$, i.e., the set of soft-labels (scores, or soft predictions). A training set, denoted by $D$, is a set of $n$ private binary-labeled data points $\{(x_1, y_1), \dots, (x_n, y_n)\} \subseteq \cX\times \cY$ drawn i.i.d. from some (arbitrary unknown) distribution $\cD$ over $\cX\times\cY$. 

Our generic construction $\cA_{\slc}$ for soft-label classification queries (Algorithm~\ref{Alg:privClass-soft} below) can be viewed as an instantiation of Algorithm~\ref{Alg:subSamp}. $\cA_{\slc}$ takes as input a \emph{private} training dataset $D$ of $n$ i.i.d. examples, and a sequence of classification queries on public data (namely, a sequence $\{x_1, \ldots, x_m\}\subseteq \cX$ of unlabeled public data). For each queried point $x_{\ell}$, where $\ell\in [m],$ $\cA_{\slc}$ responds with a score (soft-label) $\hs(x_{\ell})\in \bcY$. 
Before we formally describe our algorithm, we introduce some useful definitions. 

\begin{defn}[$\gamma$-partitions of the unit interval]\label{defn:dy-part}
Let $\gamma \in (0, 1/2]$. A $\gamma$-partition (of subintervals) for the unit interval is a partition $\cP^{\gamma}=\bigg\{\cI^{\gamma}_j:~j\in \{1, \ldots, \lceil 1/\gamma\rceil\}\bigg\}$, where $\cI^{\gamma}_j= [(j-1)\, \gamma,~~ j\, \gamma)$ for $1\leq j\leq \big\lceil 1/\gamma\big\rceil-1$, and $\cI^{\gamma}_j=\left[\left(j-1\right)\, \gamma, ~~1\right]$ for $j = \big\lceil 1/\gamma\big\rceil$.
\end{defn}
A \emph{$1/2$-shifted $\gamma$-partition $\widehat{\cP}^{\gamma}$} is a shifted version of $\cP^{\gamma}$, where each interval is shifted by half its length. We assume that $1/\gamma$ is an integer, and we remove the leftmost and rightmost half-intervals from the $1/2$-shifted $\gamma$-partitions. Hence, we get equally sized bins for the histograms in both the original and the $1/2$-shifted discretizations, and the number of bins in $\shist^{\gamma}_{\cS}$ is less than that of $\hist^{\gamma}_{\cS}$ by one. Note that this assumption is not restrictive, since for any reasonably small value $v\in [0, 1]$, one can always find a number $\gamma$ such that $1/\gamma$ is an integer, and $|v-\gamma|\approx v^2$.

\begin{defn}[$\gamma$-Histogram for a set $\cS$]
Let $\cS\subset [0, 1]$ be a finite multiset, and $\gamma\in (0, 1/2]$. A $\gamma$-histogram for $\cS$, denoted by $\hist_{\cS}^{\gamma}$, is the histogram of $\cS$ over the $\gamma$-partition $\cP^{\gamma}$, i.e.,  a mapping $\hist_{\cS}^{\gamma}:\{1, \ldots, \big\lceil 1/\gamma\big\rceil\} \rightarrow \mathbb N$ defined as
$\hist_{\cS}^{\gamma}(j)=\sum_{x\in\cS}\ind\left(x\in \cI^{\gamma}_j\right)$, for $j=1, \ldots, \big\lceil 1/\gamma\big\rceil.$
\noindent A $1/2$-shifted $\gamma$-histogram for $\cS$, denoted by $\widehat{\hist}_{\cS}^{\gamma}$, is a histogram of $\cS$ over the $1/2$-shifted $\gamma$-partition $\widehat{\cP}^{\gamma}$.
\end{defn}

\begin{defn}[Procedure $\genhist$]
Let $\genhist$ be an algorithm that takes as inputs a finite multi-set $\cS\subset [0, 1]$ and $\gamma$-partition of the unit interval (or $1/2$-shifted partition), and outputs the $\gamma$-histogram $\hist_{\cS}^{\gamma}$ (or the $1/2$-shifted $\gamma$-histogram $\widehat{\hist}_{\cS}^{\gamma}$) for $\cS$.
\end{defn}

Our algorithm $\cA_{\slc}$ can be described via two logical phases: first, it invokes a private learner $\cA_{\lpriv}$ (Algorithm~\ref{Alg:lpriv}), which uses a generic non-private learner $\cB$ to construct a private classifier $\hpriv$ (Algorithm~\ref{Alg:hpriv}). To be more specific, on input dataset $D$, $\cA_{\lpriv}$ splits $D$ into $k$ equal-sized, non-overlapping chunks $\hD_1, \ldots, \hD_k$, and runs $\cB$ on each of them separately. The resulting soft-label classifiers $h_{\hD_1}, \ldots, h_{\hD_k}$ are then used to construct a private classifier $\hpriv$. In the second logical phase, $\hpriv$ is used to answer classification queries in the form of \emph{public} feature vectors $\{x_1, \ldots, x_m\}\subseteq \cX$. In particular, answering a query $x_{\ell}$ is simply done by evaluating $\hpriv(x_{\ell})$, that is, running $\hpriv$ on input $x_{\ell}$. 

\begin{algorithm}
	\caption{$\A_{\slc}$: Private Algorithm for Soft-Label Classification Queries}
	\begin{algorithmic}[1]
		\REQUIRE Private training dataset $D\in \left(\cX\times \cY\right)^n$, ~a (non-private) learner $\cB$, sample-splitting parameter $k$, ~number of queries $m\in \mathbb N$,~ Query set $Q=\{x_1,\cdots,x_m\}\subseteq \cX$, ~procedure $\genhist$,  ~discretization parameter $\gamma \in (0, 1/2]$,~ privacy parameters $\epsilon,\delta >0$, ~ cutoff parameter $T \in [m]$
		\STATE $c\leftarrow 0$, $\lambda\leftarrow \sqrt{64T\,\log(2/\delta)}/\epsilon$, $w\leftarrow \lambda\cdot{\log\left(4\, m/\delta\right)}$, and $\aux \leftarrow \big\{\genhist, \gamma, \lambda , w \big\}$ 
		\STATE Run $\cA_{\lpriv}$ to render the classifier $\hpriv$: $\hpriv\left(\cdot~; h_1, \ldots, h_k, \aux\right)\leftarrow \cA_{\lpriv}\left(D, \cB, k, \aux\right)$
		\FOR{$\ell \in [m]$ and $c\leq T$}
		\STATE $\left(\hs(x_{\ell}), \flag \right)\leftarrow \hpriv\left(x_{\ell}~; ~\aux\right)$
		\IF{$\flag=1$}
		\STATE \textbf{if} $\hs(x_{\ell})=\bot$ \textbf{then } $c\leftarrow c+ 2$, $~\widehat w\leftarrow w+{\sf Lap}(\lambda)$  
		\STATE \textbf{else }               $c \leftarrow c+ 1$
		\ENDIF
		\STATE Output the estimated score: $\hs(x_{\ell})$
		\ENDFOR
	\end{algorithmic}
	\label{Alg:privClass-soft}
\end{algorithm}


\begin{algorithm}
	\caption{$\A_{\lpriv}$: Private Learner for Soft-Label Classification}
	\begin{algorithmic}[1]
		\REQUIRE Private training dataset $D\in \left(\cX\times \cY\right)^n$, a (non-private) learner $\cB$, sample-splitting parameter $k$, parameters for building the output classifier: $\aux=\big\{$ Procedure $\genhist$, discretization parameter $\gamma\in (0, 1/2]$, privacy noise scale $\lambda >0$, threshold $w >0\big\}$.
		
		\STATE Split the dataset $D$ into $k$ non-overlapping chunks $\hD_1,\cdots, \hD_k$ each of size $n'\triangleq n/k$. 
		\FOR{$j=1, \ldots, k$}
		\STATE $h_{j}\leftarrow \cB(\hD_{j})$
		\ENDFOR
		\STATE Output the classifier algorithm $\hpriv\left(~\cdot~; ~h_1,\ldots, h_k, ~\aux\right)$
	\end{algorithmic}
	\label{Alg:lpriv}
\end{algorithm}

Note that the first logical phase ends with $\cA_{\lpriv}$ outputting (rendering) \emph{an algorithm} $\hpriv$. This may be an unusual way of describing an algorithm's output, but we purposefully do this to be able to make statements concerning some properties of learners, where we compare the non-private learner $\cB$ and our algorithm. Since the final algorithm $\cA_{\slc}$ is itself not technically a learner\footnote{A learner outputs a classifier, i.e., a mapping from $\cX$ to $[0, 1]$, while $\cA_{\slc}$ outputs a sequence of soft-labels as responses to the queried feature vectors.}, we need to divide the execution of $\cA_{\slc}$ into two logical phases so that we can have a well-defined private learner $\cA_{\lpriv}$ which we can compare to its non-private analog $\cB$. 


\subsubsection*{The Private Classifier $\hpriv$}

The key component of algorithm $\cA_{\lpriv}$ is the classifier $\hpriv$. It is easy to see that the operation of $\hpriv$ is a tweaked version of a single iteration inside Algorithm~\ref{Alg:subSamp}. In particular, given a query $x$, classifier $\hpriv$ creates a set $\cS$ containing the $k$ soft predictions of the classifiers $h_1, \ldots, h_k$ which were produced by $\cB$ earlier (based on a partition of the original dataset). The tweak here is to create a discretization (partition) of the range $[0, 1]$ that is independent of the knowledge of $\cS$, enabling us to construct a histogram for $\cS$, and hence, proceed as in Algorithm~\ref{Alg:subSamp}. 


\paragraph{High-level description:} Our classifier is based on testing for stability first on $\hist^{\gamma}_{\cS}$ in the same fashion as done in Algorithm~\ref{Alg:subSamp}. If $\hpriv$ passes the test, i.e., $\hist^{\gamma}_{\cS}$ is sufficiently stable, it outputs the mid-point of the bin with the maximum count, and  proceeds in the usual manner. If it fails the test, then rather than directly responding with $\bot$, $\hpriv$ adds fresh noise to the stability threshold $w$ and performs another test, but this time on $\shist^{\gamma}_{\cS}$. If it passes the test, it outputs the mid-point of the interval with the maximum count. If it fails again, then it outputs $\bot$. Once $\cA_{\slc}$ receives a response, before passing on the next query to $\hpriv$, it decides based on the last response of $\hpriv$ whether or not it should increment the privacy budget counter $c$ (and if so, by how many increments), and whether or not to add fresh noise to the stability threshold $w$.



\paragraph{Dependence of the guarantees on $\gamma$:} 
Our goal is to take advantage of scenarios where there is sufficient concentration in the predictions of the $k$ classifiers, and hence an ideal setting for $\gamma$ would be about the same as the width of the concentration interval.  

\begin{algorithm}
	\caption{$\hpriv$: Private Soft-Label Classifier}
	\begin{algorithmic}[1]
		\REQUIRE Domain point (feature vector) $x\in \cX$; ~collection of soft-label classifiers $h_1, \ldots, h_k$, ~procedure $\genhist$, discretization parameter $\gamma\in (0, 1/2]$, privacy noise scale $\lambda >0$, threshold $w >0$.

        \STATE Create a multiset $\cS=\{h_1(x),\cdots, h_k(x)\}$.\label{step:multiset}
        \STATE Initialize $\flag = 0$, and let $\hist^{\gamma}_{\cS}\leftarrow \genhist\left(\cS, \gamma, \flag\right)$. 
		\STATE $\ds\leftarrow \max\left\{0, ~ \bigg(\max\limits_{v\in [1/\gamma]}\hist^{\gamma}_{\cS}(v)-{\sf second}\max\limits_{v\in [1/\gamma]}\hist^{\gamma}_{\cS}(v)\bigg)-1\right\}$, and $\wds\leftarrow\ds+{\sf Lap}\left(2\lambda\right)$.
		\IF{$\wds > \widehat{w}$}\label{step:stab-test}
        \STATE $v^*\leftarrow \arg\max\limits_{v\in [1/\gamma]}\hist^{\gamma}_{\cS}(v)$, ~$~\hs(x)\leftarrow (2\,v^*-1)\,\gamma/2$.
        \ELSE
        \STATE $~\widehat w\leftarrow w+{\sf Lap}(\lambda), ~ ~\flag=1$, and $\shist^{\gamma}_{\cS}\leftarrow \genhist\left(\cS, \gamma, \flag\right)$.
        \STATE $\ds\leftarrow \max\left\{0, ~\bigg( \max\limits_{v\in [1/\gamma-1]}\shist^{\gamma}_{\cS}(v)-{\sf second}\max\limits_{v\in [1/\gamma-1]}\shist^{\gamma}_{\cS}(v)\bigg)-1\right\}$, $\wds\leftarrow\ds+{\sf Lap}\left(2\lambda\right)$.
		\STATE \textbf{if } $\wds > \widehat{w}$ \textbf{then } $v^*\leftarrow \arg\max\limits_{v\in [1/\gamma-1]}\shist^{\gamma}_{\cS}(v)$, ~$~\hs(x)\leftarrow v^*\,\gamma$. \label{step:stab-test}
        \STATE \textbf{else } $\hs(x)\leftarrow \bot$.
      \ENDIF
        \STATE \textbf{Output} estimated score and discretization type flag $\left(\hs(x), \flag\right)$.
	\end{algorithmic}
	\label{Alg:hpriv}
\end{algorithm}

\paragraph{Per-query precision parameters:} Despite the fact that our algorithm is described for a fixed initial precision $\gamma$ for all the queries, it can be trivially extended to a more general setting where a possibly different discretization parameter $\gamma_{\ell}$ is chosen for the query $x_{\ell}$. The choice of $\gamma_{\ell}$ can then be decided, for example, based on some prior information about the quality of the query (e.g., a rough estimate of the variance in the soft prediction for the queried feature vector). In such a case, if there is a good reason to believe beforehand that soft prediction is expected to have large variance given the feature vector, then the algorithm may choose to set the discretization parameter to a larger value (leading to a coarser partition), thus potentially reducing the privacy budget (or answering more queries with the same budget). We use the same precision for all queries for simplicity. 


\begin{thm} [Privacy Guarantee of $\cA_{\slc}$]
\label{thm:privSLQ}
Algorithm \ref{Alg:privClass-soft} is $(\epsilon, \delta)$-differentially private
\end{thm}

\begin{proof}
The proof follows almost along the lines of the proofs of Corollary~\ref{cor:privSubSamp} and Theorem~\ref{thm:cPriv3}. The only difference here is that for every unstable query, we answer two sub-queries (one for each discretization), and we may pay an extra unit in the privacy budget counter. Hence, to correct for this increase, we effectively replace the factor $T$ with $2\,T$ in the noise and threshold parameters in Algorithm~\ref{Alg:seqCompPriv}, and also replace $\delta$ with $\delta/2$, which explains the extra factors of $2$ in their respective settings in $\cA_{\slc}$. 
\end{proof}

Next, we discuss a utility guarantee for Algorithm $\cA_{\slc}$. Here, we provide a very general statement that does not particularly make assumptions about the underlying learner $\cB$.  Before we state the general utility guarantee, we provide some useful definitions.

\begin{defn}[$d$-stable histogram]
Let $d\in \mathbb{N}$. Let $\cR$ be a finite domain. A histogram \mbox{$\hist_{\cS}: \cR \rightarrow \mathbb N$} for a finite set $\cS\subset \cR$ is said to be $d$-stable if 
\mbox{$\max\left\{1, \max\limits_{v\in\cR}\hist_{\cS}(v)-{\sf second}\max\limits_{v\in \cR}\hist_{\cS}(v)\right\}> d.$}
If $\vert\cR\vert = 1$, then the resulting histogram has trivially a single bin with count $\cS$. In such a case, it is trivially $d$-stable for all $d < |\cS|$.
\label{defn:hist_stable}
\end{defn}

Moreover, let $\cS_{\ell}$ denote the set constructed in Step~\ref{step:multiset} of $\hpriv$ (Algorithm~\ref{Alg:hpriv}) during its invocation for the $\ell$-th query $x_{\ell}$, i.e., $\cS_{\ell}=\{h_1(x_{\ell}), \ldots, h_k(x_{\ell})\}$.

\begin{thm}[General utility guarantee for  $\cA_{\slc}$]\label{thm:slc-gen-util}
Let $~d=32\log\left(\frac{8\, mT}{\min\left(\delta, \beta\right)}\right)\sqrt{4\,T\log(2/\delta)}/\epsilon$, and fix any $1/\gamma \in\mathbb{N}$. Define  
$\cG_0(d)=\{\ell\in[m]:~ \hist_{\cS_{\ell}}^{\gamma}~ \text{ is } d{-stable}\}$, and $\cG_1(d)=\{\ell\in[m]\setminus \cG_0(d):~ \shist_{\cS_{\ell}}^{\gamma}~ \text{ is } d{-stable}\}.$
 If $\lvert \cG_0(d)\rvert \geq m-T$, then with probability at least $1-\beta$, $\cA_{\slc}$ (Algorithm~\ref{Alg:privClass-soft}) answers all $m$ queries such that the output scores satisfy i) for $\ell \in \cG_0,$ $\hs(x_{\ell})=\gamma/2\cdot\left(2\arg\max\limits_{v\in [1/\gamma]}\hist^{\gamma}_{\cS_{\ell}}(v)-1\right)$; and ii) for $\ell \in \cG_1,$ $\hs(x_{\ell})= \gamma \cdot \arg\max\limits_{v\in [1/\gamma-1]}\shist^{\gamma}_{\cS_{\ell}}(v).$
\end{thm}

\begin{proof}
The proof of the first part follows exactly on the lines of Corollary~\ref{cor:spAcc1}, with the caveat that when we apply the sparse vector analysis, we replace $\beta$ with $\beta/2$. 

It remains to prove the second item. Combining the stability property of queries in $\cG_1$ with Theorem~\ref{thm:kstabUtil} $($with $\epsilon\leftarrow \epsilon/{2\sqrt{2T\,\log(2/\delta)}}, ~\delta\leftarrow\delta/2m, ~\beta\leftarrow \beta/4mT)$, the proof follows.
\end{proof}

Next, we provide a stronger utility guarantee for Algorithm~\ref{Alg:privClass-soft} under natural conditions on $\cB$. 
In particular, under a stability condition on $\cB$, together with a natural condition on the expected value of the soft predictions generated by an output classifier of $\cB$, we show that w.h.p., for every query, the generated score by $\cA_{\slc}$  is \emph{tightly concentrated} around the expected value of the soft-label generated by $\cB$. Hence, our algorithm provides high-confidence guarantees on the output scores when only expectation guarantees are assumed about the underlying non-private learner. That is, our algorithm not only provides responses that are almost as accurate as the expected predictions of the non-private learner, but also boosts the confidence of the generated scores. This is achieved at the expense of increasing the sample size of the non-private learner by at most a factor of $\approx \sqrt{T}\log(m/\delta)/\epsilon$, where $T$ can be much smaller than $m$ in many natural settings. As will be shown, this factor is in fact the value we set for the sample-splitting parameter $k$.
We now define our first condition, which is a slightly stronger version of on-average-RO (Replace-one) stability, but weaker than uniform stability. 
\begin{defn}[$\alpha$-$\oa$ stability of a learner]\label{defn:unif-stab}
Let $\cD$ be any distribution over $\cX\times\cY$. Let $D\sim\cD^{n}$, and $V=\{z_1', z_2', \ldots, z_n'\}\sim\cD^n$ be independent of $D$. Let $D^{(j)}$ be the dataset resulting from replacing the $j$-th entry in $D$ with $z_j'$. Let $\cB$ be a (possibly randomized) learner that, on an input dataset $D$, outputs a function $h_D:\cX\rightarrow [0, 1]$. Algorithm $\cB$ is $\alpha$-$\oa$ stable if for any $x\in\cX$, we have 
$\frac{1}{n}\sum_{j=1}^n\ex{\cB, D, V}{\lvert h_D(x) - h_{D^{(j)}}(x)\rvert ^2}\leq \alpha^2$,
where the expectation is taken over internal randomness of $\cB$, and the data points in $D, V$.
\end{defn}
Note that, in general, $\alpha$ can depend on $n$. We will not explicitly express such a dependency in the notation as long as it is clear from the context.

\paragraph{Remark:} The standard way to define stability notions is usually done with respect to some fixed loss function $\ell(h(x), y)$, for example, via a bound on $\frac{1}{n}\sum_{j=1}^n\ex{\cB, D, V}{\lvert\ell\left(h_D(x), \tilde y\right)-\ell\left(h_{D^{(j)}}(x), \tilde y\right)\rvert}$ for all $x, y$. However, under some standard and natural assumptions on the loss function, one can show that the latter implies our definition given above (up to some constant). 



Since we allow for randomized learners, we will use notation that explicitly accounts for the internal randomness of the learner for clarity. Let $R$ be a random variable that denotes the random coins of $\cB$. Hence, we can express $\cB$ as a deterministic function of $(D, R)$, where $D$ is the input dataset. This way, we can easily point to the two sources of randomness in the output classifier. Moreover, for a given realization of the random coins $r\sim R$, we let $h^{(r)}_D$ denote the output of $\cB(D, r)$.
We now state the following lemma. 

\begin{lem}\label{lem:unif-stab}
Let $\cB$ be $\alpha$-$\oa$ stable (soft-label) classification learner. Let $\cD$ be any distribution over $\cX\times\cY$. Let $\hD$ be a dataset of $n'$ i.i.d. examples from $\cX\times \cY$ drawn according to $\cD$. Then, for any fixed $x\in \cX$, we have $\pr{r\sim R,~\hD\sim\cD^{n'}}{\big\lvert h^{(r)}_{\hD}(x)-\ex{\hD\sim \cD^{n'}}{h^{(r)}_{\hD}}(x)\big\rvert \leq 4\alpha\sqrt{2n'}}\geq 3/4$,
where the probability is taken over both the random coins $r$ of $\cB$ and the dataset $\hD$. Note that the expectation inside the probability is only over $\hD$ (for fixed random coins $r$).
\end{lem}

\begin{proof}
Fix a domain point $x\in\cX$. Let $V=\{z_1', \ldots, z_{n'}'\}\sim\cD^{n'}$ that is independent of $\hD$. By Definition~\ref{defn:unif-stab}, observe that Markov's inequality implies 
$$\pr{r\sim R}{\frac{1}{n'}\sum_{j=1}^{n'}\ex{\hD, V}{\big\lvert h^{(r)}_{\hD}(x) - h^{(r)}_{\hD^{(j)}}(x)\big\rvert ^2} > 8\alpha^2}\leq 1/8.$$

Let $\cR_{\cG}=\left\{r:~ \frac{1}{n'}\sum_{j=1}^{n'}\ex{\hD, V}{\big\lvert h^{(r)}_{\hD}(x) - h^{(r)}_{\hD^{(j)}}(x)\big\rvert ^2}\leq 8\alpha^2\right\}$, that is, the set of random coins for which the mean squared differences in predictions is bounded. Note that $\pr{r\sim \cR_{\cG}}{r\sim \cR_{\cG}}\geq 7/8$. 

Fix any $r\in\cR_{\cG}$. Observe by Steele's inequality for bounding the variance \citep{steele1986efron}, we have 
$$\var{\hD}{h^{(r)}_{\hD}(x)}\leq \frac{1}{2}\sum_{j=1}^{n'} \ex{\hD, V}{\big\lvert h^{(r)}_{\hD}(x) - h^{(r)}_{\hD^{(j)}}(x)\big\rvert ^2}\leq 4\alpha^2 n'$$

Hence, by Chebyshev's inequality, we get: 
$$\pr{\hD}{\big\lvert h^{(r)}_{\hD}(x) - \ex{\hD}{h^{(r)}_{\hD}(x)}\big\rvert > 4\alpha\sqrt{2n'}}<1/8$$

Putting these together, we conclude that with probability at least $3/4$ over the random coins of $\cB$ and the randomness of the dataset $\hD$, we have $\big\lvert h^{(r)}_{\hD}(x) - \ex{\hD}{h^{(r)}_{\hD}(x)}\big\rvert \leq 4\alpha\sqrt{2n'}.$
\end{proof}

This lemma basically says that if $\cB$ is $\alpha$-$\oa$ stable, then for \emph{any} domain point $x$, the set of pairs $(r, \hD)$ for which $h^{(r)}_{\hD}(x)$ is far from its expectation over $\hD$ by more than $\approx \alpha\sqrt{n'}$ has probability measure less than $1/8$. Clearly, this result becomes useful when $\alpha \ll 1/\sqrt{n'}$. This is indeed the case for several learners, most notably for SGD. We show in Section~\ref{app:sgd} that SGD is $O(1/n')$ $\oa$ stable (implication from \cite[Theorem 3.9]{hardt2015train}) under standard assumptions on the loss function used for training. 

\paragraph{Interlude: SGD satisfies $\oa$ stability.}
\cite{hardt2015train} show that SGD \emph{with specific forms of randomization} 
satisfies uniform stability, which is stronger than our notion of $\oa$ stability. We show in Section~\ref{app:sgd} that \emph{any} standard SGD, including the (deterministic) one-pass version, satisfies $\oa$ stability. 
We recover the same bounds obtained in \cite{hardt2015train}, but for $\oa$ stability
. Our result on SGD also emphasizes the relevance of our  results in this section, since we will show that instantiating the non-private learner $\cB$ in algorithm $\cA_{\lpriv}$ with an $\oa$ stable learner provides strong accuracy guarantees for our private algorithms. 


Now, for results concerning the utility of our private algorithm, we will also assume the following condition that is commonly satisfied for a wide range of popular randomized learners. 

\begin{defn}[Consistent randomization]\label{defn:cons-rand}
Let $\cD$ be any distribution over $\cX\times\cY$. Let $n\in\mathbb N$, and $x\in \cX$ be any domain point. A randomized learner $\cB$ is said to have consistent randomization if for any pair of realizations of its random coins $r_1, ~r_2$, we have 
$\ex{D\sim\cD^n}{h^{(r_1)}_D(x)}=\ex{D\sim\cD^n}{h^{(r_2)}_D(x)}.$
In such a case, we denote the common value of the expectations as $\ex{D\sim\cD^n}{h^{\cB}_D(x)}$.
\end{defn}

Note that the popular permutation-based SGD is one good example of a randomized learner that has consistent randomization and satisfies $\oa$ stable (e.g., by the results of \cite{hardt2015train}).

\begin{cor}\label{cor:conc}
In Algorithm $\cA_{\lpriv}$, let $\cB$ be an $\alpha$-$\oa$ stable learner with consistent randomization. Let $x\in\cX$, $\beta'\in (0, 1)$, and $\cS=\{h_1(x), \ldots, h_k(x)\}$. If the sample-splitting parameter $k\geq 72\log\left(\frac{1}{\beta'}\right)$, then 
$\left\lvert \left\{v\in\cS : \lvert v - \ex{\hD\sim\cD^{n'}}{h^{\cB}_{\hD}} \rvert \leq 4\alpha\sqrt{2n'}\right\}\right\rvert \geq \frac{2k}{3}$ w.p. at least $1-\beta'$.
\end{cor}
\begin{proof}
The proof follows from Lemma~\ref{lem:unif-stab}, and the Chernoff-Hoeffding's bound (as in Theorem~\ref{thm:utilitySubAgg}).
\end{proof}

To provide stronger guarantees for our private learner $\cA_{\lpriv}$ (and hence, for $\cA_{\slc}$), we will instantiate the non-private learner $\cB$ used by $\cA_{\slc}$ with an $\oa$ stable learner with consistent randomization. 
Our guarantees involve some notion of quality of the soft predictions generated via the non-private learner $\cB$. Roughly speaking, we define the quality of a soft-label by how close it is to either \mbox{$0$ or $1$} (either in expectation, or with high probability over the randomness in  both the input training set and the learner). This notion captures the level of confidence of a classifier produced by $\cB$ about the true nature of the underlying \emph{hard} label for a given feature vector. This is because a soft-label can be viewed as an estimate for the true conditional distribution $p(y=1|x)$. 






We formally define two versions of this notion of quality for a learner with respect to a feature vector. The first (Definition~\ref{defn:wquality}) is a phrased in terms of an expectation guarantee on the soft prediction generated via a given learner for a given feature vector, whereas the second (Definition~\ref{defn:squality}) is a stronger version phrased in terms of a high-probability guarantee on the same. In particular, one can think of the second definition as a ``boosted'' version of the first . Note that these notions depend on both the learner and the queried feature vector. 

\begin{defn}[$\nu$-weak quality]\label{defn:wquality}
Let $\nu \in [0, 1/2]$. A randomized learner $\cB$ is said to have $\nu$-weak quality for a domain point $x\in\cX$ if 
 $\ex{r\sim R, ~D\sim \cD^{n}}{h_{D}^{(r)}(x)}< \nu \text{, or }\ex{r\sim R, ~D\sim \cD^{n}}{h_{D}^{(r)}(x)}> 1-\nu.$
\end{defn}

\paragraph{An instantiation for weak quality:} In scenarios with relatively low \emph{label noise}, that is, when the conditional probabilities $p(y=0|x), ~p(y=1|x)$ of the true hard label given a feature vector $x$ are sufficiently far from each other for most $x$, i.e., when $\left\lvert p(y=0|x) - p(y=1|x)\right\rvert$ is relatively large with high probability over the choice of $x$, a learner that generates sufficiently accurate soft-labels (i.e., good estimates for $p(1|x)$) will satisfy the above notion of quality with small $\nu$ for all except a small fraction of the queried feature vectors.

\noindent As an instantiation of the above scenario, consider the following simple example. Let the feature space $\cX$ be $\mathbb{R} \backslash (-c,c)$ 
for some $c>0$. For $x \in \cX$, let $y = sign(x + z)$, where $z \sim N\left(\frac{c^2}{8}\right)$. Here, $N(\sigma^2)$ represents the zero-mean Gaussian distribution with variance $\sigma^2$. From the properties of the Gaussian distribution, we have the following distribution on the labels:
\begin{equation*}
p(y=1 | x) = \begin{cases}
             > 1 - exp(-4x^2/c^2),  & \text{if } x > c \\
             < exp(-4x^2/c^2),  & \text{if } x < -c
       \end{cases} 
\end{equation*}
Therefore, we have that $\min\limits_{x \in \cX} \left\lvert p(y=1 | x) - p(y=0 | x)\right\rvert > 1 - 2e^{-4} \approx 0.96$. In particular, we have $\min\limits_{x \in \cX} \left(\max \{p(y=1 | x) , p(y=0 | x)\}\right) = 1 - e^{-4} \approx 0.98$. 

Now, if we have a soft-label learner $\cB$ that outputs an $(\alpha,\beta)$-accurate soft-label classifier $h_D$ on input dataset $D\sim \cD^{n}$, i.e., we have $\left\lvert \ex{D\sim \cD^{n}}{h_D (x)} - p(y=1|x)\right \rvert \leq \alpha$  w.p. at least $1 - \beta$ over $x \in \cX$, then we have $(\alpha + 0.02)$-weak quality for $\cB$ for at least $1-\beta$ fraction of queries $x \in \cX$. Thus, any good learner is expected to satisfy our notion of weak-quality for a significantly large fraction of queries in such scenarios. For example, assume that $\cB$ fits a logistic regression model to a sufficiently large dataset $D$ (whose labels are generate as in the above example). For model parameters $\left(\omega_D^{(0)},\omega_D^{(1)} \right)$, the expected value (w.r.t. $D$) of the soft-label prediction $\frac{exp(\omega_D^{(0)} +\omega_D^{(1)}x)}{1 + exp(\omega_D^{(0)} +\omega_D^{(1)}x)}$ will be very close to 0 or 1 w.h.p. over $x \in \cX$. 

We now define a stronger version of our notion of quality. 
\begin{defn}[$(\nu, \beta)$-strong quality]\label{defn:squality}
Let $\nu \in [0, 1/2], \beta\in (0, 1)$. A randomized learner $\cB$ is said to have $(\nu, \beta)$-strong quality for a domain point $x\in\cX$ if $h^{(r)}_{D}(x)< \nu \text{ or } h^{(r)}_{D}(x)> 1-\nu$
with probability at least $1-\beta$ over the random coins $r$ of $\cB$ and the choice of $D\sim\cD^n$.
\end{defn}

We now give the main results of this section. We view the input dataset $D$ to $\cA_{\slc}$ as a set of $n$ i.i.d. data points drawn from a distribution $\cD$. 
We first give the following lemma. 

\begin{lem}[Non-private weak quality $\Rightarrow$ private strong quality]\label{lem:strng-qual}
Let $\beta'\in (0, 1)$. Consider $\cA_{\lpriv}$ (Algorithm~\ref{Alg:lpriv}). Set $k=3\left(\sqrt{5}\lambda\log(2/\beta')+w\right)$, where $\lambda$ and $w$ are the input parameters defined in $\cA_{\lpriv}$. Let $\cB$ be any $\alpha$-$\oa$ stable learner (w.r.t. input datasets of size $n'=n/k$) with consistent randomization. Let $\nu\in [0, 1/2]$, and $x\in\cX$ be any domain point. Set $\gamma = 16\alpha\sqrt{2n'}+\nu$ (assuming, w.l.o.g. that $1/\gamma$ is an integer). Suppose that $\cB$ has $\nu$-weak quality for $x$ (w.r.t. input datasets of size $n'$). Then, w.p. at least $1-\beta'$, the output classifier $\hpriv$ of $\cA_{\lpriv}$ satisfies
$\left\lvert\hpriv(x)-\ex{\hD\sim\cD^{n'}}{h^{\cB}_{\hD}(x)}\right\rvert \leq 8\alpha\sqrt{2n'}+\nu/2.$
Consequently, $\cA_{\lpriv}$ has $(\nu', \beta')$-strong quality for $x$, where $\nu'= 8\alpha\sqrt{2n'}+\nu/2$.
\end{lem}

\begin{proof}
Let $\cS=\left(h_1(x), \ldots, h_{k}(x)\right)$ constructed in Step~\ref{step:multiset} of $\hpriv$. By Corollary \ref{cor:conc}, and the $\nu$-weak quality of $\cB$ for $x$, we get that if $k\geq 72\log(2/\beta')$ with probability at least $1-\beta'/2$, at least $2k/3$ points of $\cS$ lie in either the leftmost or the rightmost interval of $\cP^{\gamma}$. Hence, $\hist^{\gamma}_{\cS}$ is a $k/3$-stable histogram (see Definition~\ref{defn:hist_stable}). By following a similar line as of the proof of Corollary~\ref{cor:spAcc1}, we can show that if we also have $k/3\geq \sqrt{5}\lambda\log(2/\beta')+w$ with probability at least $1-\beta'/2$, then $\hpriv$ will pass the first distance-to-stability test (Step~\ref{step:stab-test}) with probability at least $1-\beta'$, and hence, output the center of the interval of $\cP^{\gamma}$ where $\ex{\hD\sim\cD^{n'}}{h^{\cB}_{\hD}(x)}$ lies. This interval, as noted above, must also be either the leftmost or the rightmost interval of $\cP^{\gamma}$. This completes the proof.
\end{proof}

Now, we give our main theorem. 


\begin{thm}[Utility guarantee for $\cA_{\slc}$ via $\oa$ stability and weak-quality]
\label{thm:mainUtil}
Let $\beta\in (0, 1)$. In  Algorithm~\ref{Alg:privClass-soft}, set
$k=136 \log\left(8\,m\,T/\min(\beta, \delta)\right)\sqrt{2T\log(2/\delta)}/\epsilon$, and $\gamma = 16\alpha\sqrt{2n'}+\nu.$ Suppose that
the learner $\cB$ is $\alpha$-$\oa$ stable learner (w.r.t. input sample size $n'=n/k$) with consistent randomization. Let $\nu\in [0, 1/2]$. Let $\cG(\nu)\triangleq\left\{\ell\in [m]:~ \cB ~\text{ has }\nu\text{-weak quality for } ~x_{\ell} \right\}.$ If $\lvert\cG(\nu)\rvert\geq m-T$, then w.p. at least $1-\beta$, $\cA_{\slc}$ answers all $m$ queries (without outputting $\bot$), and for all $\ell\in [m]$, the output score $\hs(x_{\ell})$ satisfies $\bigg\lvert\hs(x_{\ell})-\ex{\hD\sim\cD^{n'}}{h^{\cB}_{\hD}(x_{\ell})}\bigg\rvert \leq 8\alpha\sqrt{2n'}+\nu/2.$
\end{thm}
Theorem~\ref{thm:mainUtil} shows that, under the conditions of $\oa$ stability and weak quality, with high probability, for every query Algorithm $\cA_{\slc}$ is guaranteed to output a score that is close to the expected score $\ex{\hD\sim\cD^{n'}}{h^{\cB}_{\hD}(x_{\ell})}$ generated via $\cB$ regardless of the value of $\ex{\hD\sim\cD^{n'}}{h^{\cB}_{\hD}(x_{\ell})}$. Note that, as discussed earlier, there are various natural settings where the $\oa$ stability parameter $\alpha$ is $o\left(1/\sqrt{n'}\right)$.

\begin{proof}
The first part of the proof follows the same outline of that of Theorem~\ref{thm:utilitySubAgg}. 
By combining Corollary~\ref{cor:conc} with the setting for the discretization width $\gamma$, we can follow the same line of the proof of Theorem~\ref{thm:utilitySubAgg} to conclude that by setting $k$ as in the theorem statement, with probability at least $1-\beta$, all queries in $\cG(\nu)$ will be answered. Thus, by the setting we chose for $\gamma$ and the weak quality property of the queries $\ell\in \cG(\nu)$, we get $\bigg\lvert\hs(x_{\ell})-\ex{\hD\sim\cD^{n'}}{h^{\cB}_{\hD}(x_{\ell})}\bigg\rvert \leq 8\alpha\sqrt{2n'}+\nu/2.$


Now, it remains to show that the same accuracy guarantee holds for the queries not in $\cG(\nu)$. For the remainder of the proof, we will condition on the event in first part of the proof. Fix any  query $x_{\ell}$ such that $\ell\notin \cG(\nu)$. Let $\cS_{\ell}=\left(h_1(x_{\ell}), \ldots, h_{k}(x_{\ell})\right)$ as constructed in Step~\ref{step:multiset} of $\hpriv$, and $\cC_{\ell} \triangleq  \left[\ex{\hD}{h^{\cB}_{\hD}} - 4\alpha\sqrt{2n'}, \ex{\hD}{h^{\cB}_{\hD}} + 4\alpha\sqrt{2n'}\right]$. We know that this interval contains at least $2k/3$ points from $\cS_{\ell}$ (Recall that we already conditioned on the high probability event that all queries satisfy the condition in Corollary~\ref{cor:conc}). Note that in general, $\cC_{\ell}$ may intersect with at most two adjacent intervals in the partition $\cP^{\gamma}$. If the query passes the stability test from the first round, then the output score must be the mid-point of one of these two intervals. Hence, the accuracy condition in the theorem statement will hold for $x_{\ell}$ in this case. If the query does not pass the first stability test, then $\cC_{\ell}$ must be intersecting with exactly two adjacent intervals in $\cP^{\gamma}$. From the choice of $\gamma$, this implies that $\cC_{\ell}$ will be completely contained in exactly one interval in the \emph{shifted} discretization $\hcP^{\gamma}$. Hence, the output score must be the center of such interval. As a result, the accuracy condition in the theorem is satisfied in this case. This completes the proof.
\end{proof}

\paragraph{Comparison with the lower bound in \texorpdfstring{\cite{DSSUV15}}{}}
Our results in Theorem~\ref{thm:mainUtil} may seem to contradict with the lower bound in \texorpdfstring{\cite{DSSUV15}}{}, which implies that under some few assumptions, no efficient differentially private algorithm can accurately estimate the expected value of more than $n^2$ predicates (queries) over a dataset of size $n$. However, we note that the attack in the lower bound of \texorpdfstring{\cite{DSSUV15}}{} requires that the true answers of $\Omega (n^2/\gamma^2)$ queries (predicates) must arise from a ``\emph{$\gamma$-strong}'' distribution: a notion that captures how well-spread a distribution is (where $\gamma$ determines the degree of the `spread'). Interestingly, the weak quality condition in our result implies that \emph{all but $T$} queries are \emph{not} $\gamma$-strong for any $\gamma > 0$. Hence, the aforementioned lower bound does not apply to our setting.

    \subsection*{SGD and $\alpha$-$\oa$ Stability}
\label{app:sgd}

Here, we present a result establishing that  SGD satisfies $\oa$ stability. Our result does not follow directly from the work by \cite{hardt2015train} since we do not require any specific form of randomization, like random shuffling or sampling, to be performed on the dataset before or during the execution of SGD. In other words, our result applies to any standard SGD method regardless of the randomization, including the (deterministic) one-pass version of SGD. 

We provide a simple argument which, roughly speaking, reduces the $\oa$ stability of any standard SGD method to uniform stability restricted to permutation-based SGD (where data points are shuffled randomly prior to execution) for any problem class considered in \cite{hardt2015train}. Hence, one can translate all bounds on uniform stability in \cite{hardt2015train}, which are only applicable under specific randomization techniques, to  bounds on $\oa$ stability that hold without any assumptions on the nature of the randomization. This reduction does not involve any specific analysis of SGD, but goes through a simple argument that involves manipulating random variables and their expectations, and uses simple properties of i.i.d. sequences.
Since our analysis applies to a more general class of problems than soft-label classification, we will modify the notation to reflect such generality. 

We will use $\bw_D$ to denote the final output parameter when the input is $D=\left\{z_1, \ldots, z_n\right\}\sim\cD^n$. Let $V=\{z_1', \ldots, z_n'\}$ be an i.i.d. sequence that is independent of $D$. Let $z'$ be another fresh independent sample from $\cD$. As before, we define $D^{(j)}$ to be the dataset constructed by replacing the $j$-th point in $D$ with $z_j'$. Let $\hD^{(j)}$ be the dataset resulting from replacing the $j$-th entry of $D$ with $z'$. Without loss of generality, we will define stability (both $\oa$ and uniform) in terms of $\norm{\bw_D-\bw_{D'}}$ (where $D'$ denotes any dataset that differs from $D$ in exactly one point). An algorithm that takes a dataset $D$, and outputs a parameter (or a prediction rule) $\bw_D$, is $\alpha$-uniformly stable if for any \emph{fixed} pair of datasets $D, D'$ that differ in one point, we have that
$\ex{r \sim R}{\norm{\bw_D - \bw_{D'}}^2} \leq \alpha^2$,
where the expectation is taken over the random coins of the algorithm. Note that the bounds on uniform stability in \cite{hardt2015train} are derived by first obtaining a bound on $\norm{\bw_D-\bw_{D'}}$, which is then used to bound the difference in the loss function evaluated on $D$ and $D'$ in a straightforward manner (using standard properties like Lipschitz boundedness, and smoothness). Hence, the final bounds are constant factors away from the original bound on the parameters. For consistency with the way we define $\oa$ stability earlier, we prefer to focus our notation on the parameter space.  We note that our result would still apply if we define stability in terms of the loss. 
We now state our result.

\begin{theorem}
\label{thm:sgd}
Let $\cA$ be any (randomized) learner (e.g., SGD) that, on input dataset $D$, outputs a parameter vector $\bw_D$. Let $\sigma:[n]\rightarrow [n]$ be a random permutation. Let $\sigma(D)=\{z_{\sigma(1)},\ldots, z_{\sigma(n)}\}$ denote the dataset resulting from applying $\sigma$ on the indices of $D=\{z_1, \ldots, z_n\}$. If $\cA\left(\sigma\left(\cdot\right)\right)$ is $\alpha$-uniformly stable, then $\cA$ is $\alpha$-$\oa$ stable. 
\end{theorem}

\begin{proof}
Let $D\sim\cD^n$. Let $r\sim R$ denote the random coins of $\cA$ (if any). Observe that since $z'$ and $z_j, ~j\in [n]$ are i.i.d., we have

\begin{align}
\frac{1}{n}\sum_{j=1}^n\ex{r, D, D^{(j)}}{\norm{\bw_D(x) - \bw_{D^{(j)}}(x)} ^2}=\frac{1}{n}\sum_{j=1}^n\ex{r, D, \hD^{(j)}}{\norm{ \bw_D(x) - \bw_{\hD^{(j)}}(x)} ^2} \label{eqn:sum}
\end{align}

Note that the term inside the $j$-th expectation in the RHS is a function of $D, \hD^{(j)}$. 
For every $j \in [n]$, create a new dataset $\tD^{(j)}$ by replacing the first entry in $\hD^{(j)}$ (i.e., $z_1$) by $z_j$, i.e., $\tD^{(j)}=\{z_j, z_2, z_3, \ldots, z_{j-1}, z', z_{j+1}, \ldots, z_n\}$. Also, construct another dataset $D_*^{(j)}$ which is identical to $D$ except that $z_1$ and $z_j$ are swapped. Observe that $\tD^{(1)} = \hD^{(1)}$, and $D_*^{(1)} = D$. Now, for $j = 2, \ldots, n$, we argue that the pair $(D_*^{(j)}, \tD^{(j)})$ is identically distributed to the pair $(D, \hD^{(j)})$. Note that the only relevant random variables, paired according to their respective positions, are $(z_1, z_1),~ (z_j, z')$ (from $(D, \hD^{(j)})$) and $(z_j, z_j), ~(z_1, z')$ (from $(D_*^{(j)}, \tD^{(j)})$), respectively. It is easy to see that the first group is identically distributed to the second, as the joint distribution $p\left((z_1, z_1), (z_j, z')\right)=p(z_1)p(z_j)p(z')=p\left((z_j, z_j), (z_1, z')\right)$. 

For every $j \in [n]$, we can now safely replace $D$ with $D_*^{(j)}$, and $\hD^{(j)}$ with $\tD^{(j)}$. Note that all $D_*^{(j)}$ have the same data points as $D$, except they may be shuffled. Also, observe that all $\tD^{(j)}$ contain the same data points $z_2, \ldots, z_n$ and $z'$, with $z'$ being the $j$-th entry of $\tD^{(j)}$.  Hence, from equation~\ref{eqn:sum}, we have:
\begin{align*}
\frac{1}{n}\sum_{j=1}^n\ex{r, D, \hD^{(j)}}{\norm{ \bw_D(x) - \bw_{\hD^{(j)}}(x)} ^2}&=\frac{1}{n}\sum_{j=1}^n\ex{r, D_*^{(j)}, \tD^{(j)}}{\norm{ \bw_{D_*^{(j)}}(x) - \bw_{\tD^{(j)}}(x)} ^2}\\
&=\ex{j\leftarrow [n]}{\ex{r, D_*^{(j)}, \tD^{(j)}}{\norm{ \bw_{D_*^{(j)}}(x) - \bw_{\tD^{(j)}}(x)} ^2}}
\end{align*}
Here, the outer expectation is taken over $j$ drawn uniformly from $[n]$ (in place of the average over the indices). Note that the location of $z'$ is hence uniformly random over $[n]$. 
Let $\sigma$ be a random permutation over $[n]$. Note that since the \emph{marginal} distribution of $\sigma(1)$ is uniform over $[n]$, we can write 
the above expectation as 

$$\ex{j\leftarrow\sigma(1)}{\ex{r, D_*^{(j)}, \tD^{(j)}}{\norm{ \bw_{D_*^{(j)}}(x) - \bw_{\tD^{(j)}}(x)} ^2}}$$

For any $j\in [n]$, let $\sigma_{|j}=\left(\sigma(2), \ldots, \sigma(n)\right)$ denote the permutation induced by $\sigma$ on the positions of the dataset ranging from $2$ to $n$ \emph{conditioned on the event that $\sigma(1)=j$}. Since \emph{any} permutation of an i.i.d. sequence is still i.i.d. (by exchangeability of i.i.d. random variables), then for any fixed $j\in [n]$, the pair $\left(D_*^{(j)}, ~\tD^{(j)}\right)$ is identically distributed to $\left(\sigma_{|j}\left(D_*^{(j)}\right), ~\sigma_{|j}\left(\tD^{(j)}\right)\right)$. Hence, the above expression can be written as 

$$\ex{j\leftarrow\sigma(1)}{\ex{\sigma_{|j}}{\ex{r, D_*^{(j)}, \tD^{(j)}}{\norm{ \bw_{\sigma_{|j}\left(D_*^{(j)}\right)}(x) - \bw_{\sigma_{|j}\left(\tD^{(j)}\right)}(x)} ^2}}}$$

Note that when $j\leftarrow \sigma(1)$, we have $\sigma_{|j}\left(D_*^{(j)}\right)=\sigma\left(D_*^{(1)}\right)$ and $\sigma_{|j}\left(\tD^{(j)}\right)=\sigma\left(\tD^{(1)}\right)$. Hence, using this fact, combining the outer expectations, and swapping the order of expectations, we get   
$$\ex{r, D_*^{(1)}, \tD^{(1)}}{\ex{\sigma}{\norm{ \bw_{\sigma\left(D_*^{(1)}\right)}(x) - \bw_{\sigma\left(\tD^{(1)}\right)}(x)} ^2}}$$

Note that $D_*^{(1)}$ and $\tD^{(1)}$ may differ only in their first position (the first entry in $D_*^{(1)}$ is $z_1$ whereas the first entry in $\tD^{(1)}$ is $z'$). Thus, since $\cA\left(\sigma\left(\cdot\right)\right)$ is $\alpha$-uniformly stable, the inner expectation is bounded by $\alpha^2$, proving that $\cA$ is $\alpha$-$\oa$-stable.

\end{proof}

	\bibliographystyle{alpha} 

    \bibliography{reference}

\end{document}